\documentclass[Afour,sageh,times]{sagej}

\allowdisplaybreaks

\usepackage{moreverb,url}
\usepackage{stackengine}
\usepackage{mathtools}
\usepackage{bm}
\usepackage{float}
\usepackage{algorithm, algpseudocode}
\usepackage{amsmath}
\usepackage{amssymb}
\usepackage{amsthm}
\usepackage{enumitem}
\usepackage{csquotes}
\usepackage{footnote}
\makesavenoteenv{tabular}
\makesavenoteenv{table}
\usepackage[]{threeparttable}
\usepackage{breqn}
\usepackage{subcaption}

\usepackage[colorlinks,bookmarksopen,bookmarksnumbered,citecolor=red,urlcolor=red]{hyperref}

\newcommand\BibTeX{{\rmfamily B\kern-.05em \textsc{i\kern-.025em b}\kern-.08em
    T\kern-.1667em\lower.7ex\hbox{E}\kern-.125emX}}

\newtheorem{theorem}{Theorem}
\newtheorem{remark}{Remark}

\newtheorem{definition}{Definition}

\DeclareMathOperator*{\argmin}{arg\,min}

\usepackage{color}
\definecolor{dkgreen}{rgb}{0,0.6,0}
\definecolor{gray}{rgb}{0.5,0.5,0.5}
\definecolor{mauve}{rgb}{0.58,0,0.82}
\usepackage{ifthen,version}
\newboolean{include-notes}

\usepackage{multirow}
\usepackage{multicol}



\def\volumeyear{2024}

\begin{document}

\runninghead{Sun et al.}

\title{Mixed strategy Nash equilibrium for crowd navigation}

\author{Max Muchen Sun\affilnum{1}, Francesca Baldini\affilnum{2}, Katie Hughes\affilnum{1}, Peter Trautman\affilnum{2}, and Todd Murphey\affilnum{1}}

\affiliation{\affilnum{1}Department of Mechanical Engineering, Northwestern University, Evanston, IL 60208, USA
    \affilnum{2}Honda Research Institute, San Jose, CA 95134, USA}

\corrauth{Max Muchen Sun, Department of Mechanical Engineering, Northwestern University, Evanston, IL 60208, USA}

\email{msun@u.northwestern.edu}

\begin{abstract} 
    Robots navigating in crowded areas should negotiate free space with humans rather than fully controlling collision avoidance, as this can lead to freezing behavior. Game theory provides a framework for the robot to reason about potential cooperation from humans for collision avoidance during path planning. In particular, the mixed strategy Nash equilibrium captures the negotiation behavior under uncertainty, making it well suited for crowd navigation. However, computing the mixed strategy Nash equilibrium is often prohibitively expensive for real-time decision-making. In this paper, we propose an iterative Bayesian update scheme over probability distributions of trajectories. The algorithm simultaneously generates a stochastic plan for the robot and probabilistic predictions of other pedestrians' paths. We prove that the proposed algorithm is equivalent to solving a mixed strategy game for crowd navigation, and the algorithm guarantees the recovery of the global Nash equilibrium of the game. We name our algorithm Bayesian Recursive Nash Equilibrium (BRNE) and develop a real-time model prediction crowd navigation framework. Since BRNE is not solving a general-purpose mixed strategy Nash equilibrium but a tailored formula specifically for crowd navigation, it can compute the solution in real-time on a low-power embedded computer. We evaluate BRNE in both simulated environments and real-world pedestrian datasets. BRNE consistently outperforms non-learning and learning-based methods regarding safety and navigation efficiency. It also reaches human-level crowd navigation performance in the pedestrian dataset benchmark. Lastly, we demonstrate the practicality of our algorithm with real humans on an untethered quadruped robot with fully onboard perception and computation.
\end{abstract}


\maketitle

\section{Introduction} 

\begin{figure*}[ht!]
    \centering
    \includegraphics[width=\textwidth]{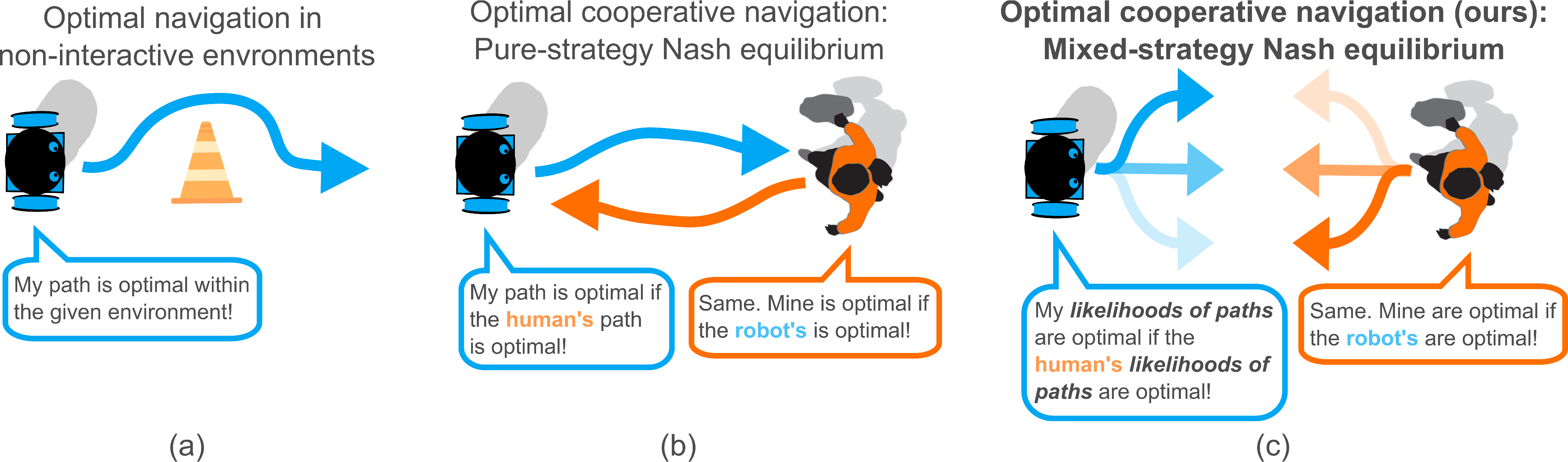}
    \captionsetup{justification=justified}
    \caption{Comparison of optimality criterion in different navigation frameworks. (a) In traditional robot navigation, the robot makes optimal decisions, such as minimizing the risk of collision, in a given and non-interactive environment; (b) Cooperative navigation finds optimal cooperative decisions for both the robot and the human. With the pure strategy Nash equilibrium model, the robot expects deterministic actions from humans, which is too assertive considering the uncertain nature of human behavior; (c) Our cooperative navigation framework uses mixed strategy Nash equilibrium as the optimality criterion, which finds probabilities of actions that represent the optimal cooperation strategies between the robot and human. This model maintains uncertainty during the interaction.}
    \label{fig:high_level_compare}
\end{figure*}

The ability to navigate fluently and safely in human-populated spaces is becoming increasingly crucial for deploying robots in real-world environments. Examples include autonomous driving in populated areas~\citep{bai_intention-aware_2015}, service robots at mass events~\citep{singh_use_2021}, and industrial robots operating alongside human workers~\citep{bansal_bayesnash_2022}. This problem of safe and efficient robot navigation in human crowds with minimal disruption to humans is often referred to as \emph{crowd navigation}. 

Human behavior studies~\citep{murakami_mutual_2021,bacik_lane_2023} reveal that humans anticipate each other's collision avoidance behavior when navigating crowds. Empirical results further indicate that navigation algorithms that predict human cooperation for collision avoidance during planning could improve both safety and navigation efficiency~\citep{trautman_robot_2015}. On the other hand, coupling the prediction of cooperative human behavior with motion planning is different from multi-agent planning, such as through crowd motion simulation~\citep{helbing_social_1995, van_den_berg_reciprocal_2011}, since humans do not follow pre-defined action rules. Instead, both the humans and the robot should be modeled as optimal planners whose individual objectives depend on others` actions~\citep{sadigh_planning_2018}. As a result, the simultaneous prediction of cooperative human behavior and robot motion planning becomes the outcome of assessing the coupled optimal actions of all the agents. This design leads to the application of game theory models in crowd navigation, specifically Nash equilibrium~\citep{nash_equilibrium_1950} as an optimality criterion for the coupled prediction and planning process.

In game theory, each agent tries to optimize a game strategy---a plan of action---for their individual objective of the game. However, each agent's individual objective depends on other agents' strategies, so no agent can optimize their strategy in isolation. Nash equilibrium describes a balanced state among all agents' strategies, where no agent wants to change their strategy unless others also change. In other words, each agent's Nash equilibrium strategy is optimal given other agents' Nash equilibrium strategies. Crowd navigation can be modeled as a game where the robot and pedestrians negotiate free space to ensure the joint safety of all agents and also minimize the compromise on their individual navigation plan. By finding the Nash equilibrium of such a game formula, the robot can plan a safe path while anticipating humans' reactions, leveraging potential human cooperation to avoid over-cautious collision avoidance behaviors.

There are two kinds of strategies that can lead to a Nash equilibrium: pure strategy and mixed strategy. We are interested in the latter. Pure strategy refers to choosing deterministic and specific actions, while mixed strategy refers to sampling actions from a probability distribution of possible pure strategies. Crucially, John Nash proved that not all games have a pure strategy Nash equilibrium, but at least one mixed strategy Nash equilibrium exists for all games~\citep{nash_non-cooperative_1951}. Our interest in mixed strategy Nash equilibrium stems from not only its mathematical rigor but also its non-deterministic nature. Prediction based on pure strategy Nash equilibrium would expect humans to follow the paths exactly as predicted, which is too assertive given the uncertain nature of human behavior. Mixed strategy Nash equilibrium, however, maintains such behavioral uncertainty during the robot’s decision-making. A conceptual comparison of traditional navigation framework, cooperative navigation with pure strategy Nash equilibrium, and with mixed strategy Nash equilibrium is shown in Figure~\ref{fig:high_level_compare}.

Despite its potential, applying mixed strategy Nash equilibrium to crowd navigation faces several challenges. First, computing mixed strategy Nash equilibrium is often considered impractical due to its computation burden---even for a 3-agent game, the computation can be close to NP-complete~\citep{daskalakis_complexity_2009}. Second, existing Nash equilibrium frameworks often focus on discrete actions and a finite number of strategies (e.g., poker). However, crowd navigation studies physical agents navigating in a continuous space. Third, the practicality of Nash equilibrium depends on how well the game design matches human behavior, yet there is no consensus on a game-theoretic model for crowd navigation.

In this work, we propose a simple iterative Bayesian update scheme for cooperative crowd navigation with mixed strategies. We formally prove that the algorithm guarantees the recovery of the global Nash equilibrium of a mixed strategy game suitable for crowd navigation. We name our algorithm Bayesian Recursive Nash Equilibrium (BRNE). In the BRNE game, each agent aims to minimize the expected risk of collision with other agents, while also minimizing the deviation of the optimal navigation strategy from a nominal navigation strategy. Furthermore, we propose a sampling-based model predictive crowd navigation framework based on BRNE, with the nominal strategies characterized based on Gaussian processes. The proposed navigation framework has a lower time complexity with respect to the number of agents compared to the state-of-the-art game-theoretic planners and runs in real-time on a laptop CPU and a low-power embedded computer. We further integrate the algorithm into an untethered quadruped robot and demonstrate the algorithm's practicality with real humans using fully onboard perception and computation. 

Our work diverges from existing game-theoretic crowd navigation methods. Existing methods follow a top-down approach, where no specification of the game objective structure and the game decision space is made in advance---these components will be specified later for applications such as crowd navigation. While the solution methodologies would apply to any game, not just crowd navigation game, this generality comes at the price of local optimality\footnote{Global Nash equilibrium solvers for arbitrary games is computationally intractable.} and unaffordable computation cost for real-time decision-making in human crowds. In contrast, we take a bottom-up approach, formulating a specific game based on unique behavioral features in crowd navigation, such as goal-oriented and collision avoidance-driven behaviors. This bottom-up approach provides a sufficient structure to develop a planner with game theory guarantees, including recovering global mixed strategy Nash equilibrium, without compromising the computation efficiency for real-time decision-making, as demonstrated in our experiments.

The rest of the paper is organized as follows: Section~\ref{sec:related_work} reviews the literature on crowd navigation and discusses the key difference between our method and existing methods. We introduce the iterative Bayesian update scheme and the model predictive crowd navigation framework in Section~\ref{sec:bayesian_updating}. Then we show how the algorithm guarantees the convergence to a global Nash equilibrium of a mixed strategy crowd navigation game in Section~\ref{sec:optimality_analysis}, with extra safety property analysis. Section~\ref{sec:evaluation} contains the details of the evaluation results and the real-world hardware demonstration. Lastly, we conclude the paper and provide further discussion in Section~\ref{sec:conclusion_discussion}.

\section{Related work} \label{sec:related_work}

\subsection{Early work on robot navigation in crowds}
Roboticists have been investigating navigation in human environments since the 1990s. Two landmark studies were the RHINO~\citep{burgard_interactive_1998} and MINERVA~\citep{thrun_probabilistic_2000} experiments, where robotics systems were deployed in museums to provide tour-guide to thousands of visitors. Additional work for tour-guide robots was also conducted later, such as Robox~\citep{siegwart_robox_2003}, Mobot~\citep{nourbakhsh_mobot_2003}, Rackham~\citep{clodic_rackham_2006}, and CiceRobot~\citep{chella_perception_2009}. A comprehensive review of the history of crowd navigation can be found in~\citep{mavrogiannis_core_2023}.

These works use conventional indoor navigation stacks such as the dynamic window approach~\citep{fox_dynamic_1997}, where humans are modeled as non-reactive obstacles. While these methods are sufficient when the robot interacts with sparse crowds, the robot's navigation efficiency is limited because of safety concerns. For example, the robot's speed is intentionally limited to avoid constant emergency stopping~\citep{nourbakhsh_mobot_2003}.  

\subsection{Prediction-then-planning methods}
The limitations of conventional navigation methods in human crowds motivate researchers to develop human-aware navigation methods. One prevalent framework in this category is to predict human motion and plan robot actions to avoid hindering human motion, we name such framework as \emph{prediction-then-planning}. Human motion can be predicted in the form of cost maps, such as from inverse reinforcement learning~\citep{ziebart_planning-based_2009,luber_socially-aware_2012} or unsupervised learning~\citep{henry_learning_2010}. Human motion can also be directly predicted in forms of trajectories with deep learning methods, such as using long-short term memory models (LSTM)~\citep{alahi_social_2016}, generative adversarial networks (GAN)~\citep{gupta_social_2018} or graph-based models~\citep{salzmann_trajectron_2020}. We refer readers to~\citet{rudenko_human_2020} for a comprehensive review of human trajectory prediction. Note that, in~\citet{scholler_what_2020}, the authors report that a simple constant velocity model could outperform deep learning-based trajectory prediction methods, arguing that existing neural network architectures are insufficient to capture interpersonal interactions. In addition, another thread of research aims to develop fast reactive planners around dynamic obstacles, such as through field representations~\citep{huber_avoiding_2022} or graph search~\citep{cao_dynamic_2019}.

Planning on top of human motion prediction also introduces the necessity of uncertainty-aware planning: motion planning methods that consider predictive uncertainty in human motion. In~\citet{du_toit_robot_2012}, a closed-loop belief update of dynamic obstacles' states is incorporated into receding horizon planning, which reduces the anticipated obstacle uncertainty and generates less conservative navigation behavior. A confidence-ware motion planning framework is proposed in~\citet{fridovich-keil_confidence-aware_2020} that maintains a Bayesian belief of the human motion prediction confidence with online observations. The planner can be robust against unexpected human behavior by explicitly modeling prediction confidence. Other works have also investigated robust motion planning with uncertain human motion prediction. In particular,~\citet{nishimura_risk-sensitive_2020} introduces a decoupled framework that combines learning-based trajectory prediction~\citep{salzmann_trajectron_2020} with risk-sensitive model predictive control~\citep{nishimura_sacbp_2021}. Lastly, a flipped approach is introduced in~\citet{nishimura_rap_2023} in the context of autonomous driving, where a risk-sensitive trajectory prediction approach is proposed for robust planning.

\subsubsection*{Limitations of prediction-then-planning methods} While exhibiting more compliant navigation behavior alongside humans, decoupled prediction and planning is limited by its core assumption---that the robot's action will not influence humans' actions. Failing to account for human reaction could lead to robot actions that surprise humans, who in turn react out of the robot's expectation, resulting in short oscillatory interaction, often referred to as ``reciprocal dance''~\citep{feurtey_simulating_2000}. Failing to account for human reaction in uncertainty quantification could also lead to over-conservative navigation robot behavior---without incorporating human reaction to belief updating during planning, the predictive uncertainty could lead the robot to consider all viable paths are unsafe and the only safe option is to stay still, a phenomenon often referred to as the ``freezing robot problem''~\citep{trautman_robot_2015}.

\subsection{Coupled prediction and planning}

Given the limitations of decoupled prediction and planning, it becomes necessary to lift the assumption that the robot does not interfere with human intents. Thus, an alternative framework, named \emph{coupled prediction and planning}, seeks to simultaneously plan robot actions and predict human actions. The seminal work \citep{trautman_unfreezing_2010} shows that cooperative collision avoidance (CCA)---where humans and robots collectively make decisions for collision avoidance---is crucial for avoiding artifacts such as the ``reciprocal dance'' and the ``freezing robot problem''. The importance of CCA in dense crowds is experimentally verified in a behavioral study~\citep{murakami_mutual_2021}.

In~\citet{trautman_robot_2015}, an individual's intent is modeled as a Gaussian process and CCA is modeled as a joint decision-making process by coupling Gaussian processes through a collision avoidance-based likelihood function. The statistical optimality of coupled Gaussian processes is further investigated in~\citet{trautman_sparse_2017} and \citet{trautman_real_2020}. In~\citet{sun_move_2021}, a distribution space crowd navigation model is proposed, with agent intent modeled as a distribution of trajectories. While similar to the mixed strategy model in this work, the model in~\citet{sun_move_2021} is not a game-theory model but is instead a joint decision-making model similar to~\citet{trautman_robot_2015}. Topology-based abstractions are also used for modeling CCA, such as through braids~\citep{mavrogiannis_socially_2017, mavrogiannis_multi-agent_2019} or topological-invariance~\citep{mavrogiannis_social_2018, mavrogiannis_hamiltonian_2021, mavrogiannis_winding_2023}. CCA is also studied from the perspective of opinion dynamics~\citep{bizyaeva_nonlinear_2023}, with a focus on breaking deadlock situations such as ``reciprocal dance''~\citep{cathcart_proactive_2023}.  

Another group of works focuses on the implicit modeling of cooperative collision avoidance. Implicit CCA models are often obtained from real-world data or simulated human pedestrians---the availability of data will impact the choice of modeling technique. Given the cost of human data collection, implicit modeling with human data cannot afford directly training robot navigation policies with humans; it uses techniques such as imitation learning~\citep{kim_maximum_2013} or inverse reinforcement learning~\citep{kim_socially_2016, kretzschmar_socially_2016}. Aside from collecting human data, an alternative solution is to simulate cooperative navigation agents using decentralized multi-agent collision avoidance methods, such as~\citet{helbing_social_1995} and \citet{van_den_berg_reciprocal_2011}. The benefit of implicit CCA modeling from simulated data is that the navigation policies can be trained directly with the reactive agents using reinforcement learning methods, such as in~\citet{chen_socially_2017}, \citet{chen_crowd-robot_2019}, \citet{liang_crowd-steer_2020}, and \citet{liu_social_2021}. However, the sim-to-real transfer of such policies remains an open challenge, mainly due to the distribution shifts between simulated crowd behavior and real-world human behavior. It is also unclear whether the reinforcement learning policies explicitly predict human cooperation during planning, instead of treating humans merely as dynamic obstacles---the latter case falls into the category of the prediction-then-planning framework.

Note that there is a relevant group of works on social convention-aware navigation. The goal is to take into account social conventions during motion planning, such as people's tendency to walk in groups~\citep{wang_group-based_2022, sathyamoorthy_comet_2022} or people's preference for not moving when forming certain geometrical structures~\citep{sathyamoorthy_frozone_2020}. Though conceptually similar, we consider these works to be essentially different from coupled prediction and planning, because the social conventions are defined a priori. As a result, these planners do not necessarily consider humans' \emph{reactions} to the robot's actions.

\subsubsection*{Limitations of coupled prediction and planning methods} 

The main limitation of existing coupled prediction and planning methods is the lack of a rigorous cooperation model since humans and robots cooperate only through mutual observations instead of explicit communication. Cooperation is often granted in coupled prediction and planning methods. For example, methods that formulate crowd navigation as a joint decision-making problem assume that humans and robots share the same joint decision-making objective, making these methods more similar to a centralized multi-agent planning framework. We argue that each agent in crowd navigation should be modeled as an independent decision-making individual. Cooperative collision avoidance is an emergent phenomenon generated by agents' desire to optimize their individual objectives, which depends on other agents' actions. This insight naturally leads to the application of game theory to crowd navigation.

\subsection{Game-theoretic planning}

Even though game-theoretic planning falls into the category of coupled prediction and planning, we separately discuss this body of work here, given its close connection to our work. The key difference between game-theoretic methods and other coupled prediction and planning methods is that game theory assumes each agent makes individual optimal decisions and provides a rigorous optimality criterion for decision-making---the concept of equilibrium. The most commonly used equilibrium notions are Nash equilibrium~\citep{nash_equilibrium_1950} and Stackelberg equilibrium~\citep{von_stackelberg_market_2011}. Stackelberg equilibrium, in general, has lower computation complexity since it enforces a leader-follower structure. On the other hand, Nash equilibrium is considered the more rigorous and natural notion for crowd navigation, since it assumes equal status for all agents and all agents make decisions simultaneously.

To apply game theory models, crowd navigation is often formulated as a dynamic game: agents, with states governed by certain dynamics, optimize control policies as pure strategies. One of the pioneering works in this direction is~\citet{sadigh_planning_2018}, where the interaction between a human-driven vehicle and an autonomous vehicle is modeled as a dynamic game. Since then,  dynamic game solvers with better numerical efficiency have been developed. In \citet{fridovich-keil_efficient_2020}, the authors formulate a general-sum dynamic game for crowd navigation, where the local Nash equilibrium is solved by combining an iterative best response scheme and linear-quadratic regulator. In \citet{lecleach_algames_2022}, a fast augmented Lagrangian solver is proposed, which converges locally to the generalized Nash equilibrium of a dynamic game and supports real-time model predictive control for autonomous driving. However, pure strategy Nash equilibrium models expect humans to react exactly as predicted, which is too assertive for interacting with humans. Furthermore, existing pure strategy game-theoretic planners can only approximate a locally optimal solution since computing the Nash equilibrium of a generalized game is computationally intractable~\cite {daskalakis_complexity_2009}. 

Probabilistic variants of pure strategy Nash equilibrium have been proposed to address the lack of flexibility. In~\citet{williams_best_2018}, an iterative best response scheme is combined with a model predictive path integral (MPPI) control framework to approximate cooperative stochastic control policies between two autonomous vehicles. However, this approach does not formally guarantee recovering a mixed strategy Nash equilibrium. In \citet{mehr_maximum-entropy_2023}, a stochastic dynamic game formula is proposed to lift the strict rationality assumption of Nash equilibrium and instead assumes bounded rationality of agents. The authors propose a new notion of equilibrium named Entropic Cost Equilibrium (ECE) and show that ECE is equivalent to the mixed strategy Nash equilibrium of a maximum entropy game. The proposed equilibrium formula also enables inverse inference of interaction policy from observations. In \citet{so_mpogames_2023} and \citet{so_multimodal_2022}, the multimodality in stochastic games is investigated from the perspective of linear-quadratic games and partially observable Markov decision process (POMDP), where the uncertainty originated from the partially observable objectives of non-ego agents. In~\citet{peters_learning_2022}, an explicit mixed strategy game formula is proposed, allowing agents to simultaneously optimize multiple multi-agent pure strategies, but this formula does not recover the mixed strategy Nash equilibrium. 

Lastly, game theory models can be combined with other methods for online adaptation during crowd navigation. For example, Nash equilibrium can be used to infer other agents' internal states, such as altruism. In \citet{schwarting_social_2019} the pure strategy Nash equilibrium model is used as an inference model for online estimation of human drivers' social value orientation, leading to the improved prediction accuracy of human driver actions during the interaction. Similar ideas are also explored in~\citet{peters_inference-based_2020}, \citet{le_cleach_lucidgames_2021}, and \citet{bansal_bayesnash_2022}. Another example of bootstrapping adaptive decision-making with game-theoretic models is~\citet{peters_contingency_2024}, where the game-theoretic model provides candidates for contingency planning.

\subsubsection*{Limitations of game-theoretic planning methods}

There are two major limitations of existing game-theoretic planning methods for crowd navigation. The first limitation is the high computation cost, which prevents existing methods from being applied to real-time navigation in human crowds. For example, as one of the fastest pure strategy dynamic game solvers, ALGAMES~\citep{lecleach_algames_2022} still suffers from a time complexity of $O(M^3)$ with $M$ being the number of players, thus can only perform real-time inference with no more than 2 human pedestrians. The computation of mixed strategy Nash equilibrium is even more burdensome. In~\citet{mehr_maximum-entropy_2023}, the method is also only demonstrated with no more than 2 human agents. In~\citet{peters_learning_2022}, the computation of mixed strategies is offloaded to a neural network trained offline and this approach is only demonstrated for two-agent interaction. The second limitation is the lack of flexibility in uncertainty representation. In \citet{mehr_maximum-entropy_2023}, the solution is based on the Linear-Quadratic-Gaussian formula. In~\citet{peters_learning_2022}, the mixed strategy model is limited to a small number of pure strategies, instead of a probability distribution. 

We argue that these two limitations are largely due to existing methods' ``top-down'' approach. In the ``top-down'' approach, multi-agent interaction is formulated and solved as a generalized game, where no specification of agent objective is made. The agent objective will be specified later for specific applications such as crowd navigation. While this ``top-down'' approach could be applied for arbitrary types of multi-agent interaction beyond crowd navigation, it suffers from high computation complexity and often requires narrowing down strategy representation (e.g., Gaussian distributions). Our work, however, takes a ``bottom-up'' approach---we tailor a stochastic game formula specifically for crowd navigation, based on unique properties in crowd navigation, such as the insight that each agent has an individual goal-oriented objective while being coupled with other agents through a collision avoidance objective. This ``bottom-up'' approach allows us to compute mixed strategy Nash equilibrium in closed form, enabling real-time inference with more agents (up to 8 agents on a laptop CPU, 5 agents on a low-power embedded computer). In addition, our algorithm and proofs support arbitrary types of distribution and collision avoidance cost functions, which further enriches the representation flexibility of our model.

\section{Bayesian update for crowd navigation} \label{sec:bayesian_updating}

In this section, we introduce the iterative Bayesian update scheme for crowd navigation. We will focus on the algorithm description and the model predictive navigation framework, leaving the formal properties of the algorithm, such as the guaranteed convergence to a Nash equilibrium, in Section~\ref{sec:optimality_analysis}. For the same reason, we leave the introduction of the complete game-theoretic formula in Section~\ref{sec:optimality_analysis} as well.

\subsection{Notations and definitions}
We assume there are $N$ agents, including the robot and the pedestrians, in a two-dimensional navigation environment. We start by defining a set of unique indices $\mathcal{I}=\{1, 2, \dots, N\}$ for all agents, where the first index $1$ is reserved for the robot. The state space of each agent is denoted as $\mathcal{X}\subset\mathbb{R}^{2}$, as we are primarily interested in agents' planar positions. 

While we will fully introduce the game theory formula in Section~\ref{sec:optimality_analysis}, we introduce the concept of strategy here. A strategy describes an agent's decision-making process. In the context of navigation, we assume agents make decisions regarding trajectories, which outline where they intend to travel over a given time period $[0, T]$.

\begin{definition}{(Pure strategy)}
    A pure strategy $s(t)$ is defined as a trajectory, which is specified as a continuous mapping from time to a state in the navigation space:
    \begin{align}
        s(t) :[0,T]\mapsto\mathcal{X} .
    \end{align}
\end{definition}

\begin{definition}{(Strategy space)}
    The strategy space is defined as the space of all feasible pure strategies that an agent may consider, we denote it as $\mathcal{S}$. 
\end{definition}

\begin{definition}{(Mixed strategy)}
    A mixed strategy is a probability distribution over the strategy space $\mathcal{S}$, represented as a probability density function $p(s)$:
    \begin{align}
        p(s):\mathcal{S}\mapsto\mathbb{R}_0^{+} \\
        \int_{\mathcal{S}} p(s) ds = 1.
    \end{align}
\end{definition}

\begin{definition}{(Mixed strategy space)}
    The mixed strategy space, denoted as $\mathcal{P}$, is the space of all probability density functions over the strategy space $\mathcal{S}$.
\end{definition}

\begin{definition}{(Nominal mixed strategy)}
    Each agent is assumed to have a nominal mixed strategy before interacting with other agents, denoted as $p^\prime_i(s)$, with $i$ being the agent index.
\end{definition}
\begin{remark} \label{remark:nominal_strategy}
    We assume the nominal mixed strategy represents the agent's intent without the presence of other agents. Thus, it does not reflect any collision avoidance behavior.
\end{remark}

\begin{definition}{(Collision risk)} \label{def:risk_function}
    Collision risk is defined as a function $r(s_1, s_2):\mathcal{S}\times\mathcal{S}\mapsto \mathbb{R}_0^{+}$, which evaluates the collision risk between two pure strategies (trajectories). 
\end{definition}

\begin{definition}{(Expected collision risk)}
    Given agent $i$'s mixed strategy $p_i(s)$, the expected risk of another pure strategy $s(t)$ colliding with agent $i$ is:
    \begin{align}
        \mathbb{E}_{p_i}[r](s) = \int_{\mathcal{S}} r(s,\xi) p_i(\xi) d\xi .
    \end{align}
\end{definition}

\begin{definition}{(Joint expected collision risk)}
    The joint expected collision risk between two agents is defined as the joint expectation of collision risk with respect to their mixed strategies:
    \begin{align}
        \mathbb{E}_{p_i,p_j}[r]
        = \int_{\mathcal{S}}\int_{\mathcal{S}} p_{i}(s_i) p_{j}(s_j) r(s_i,s_j) ds_i ds_j . \label{eq:joint_expect_risk}
    \end{align}		
\end{definition}

\begin{figure*}[ht!]
    \centering
    \includegraphics[width=\textwidth]{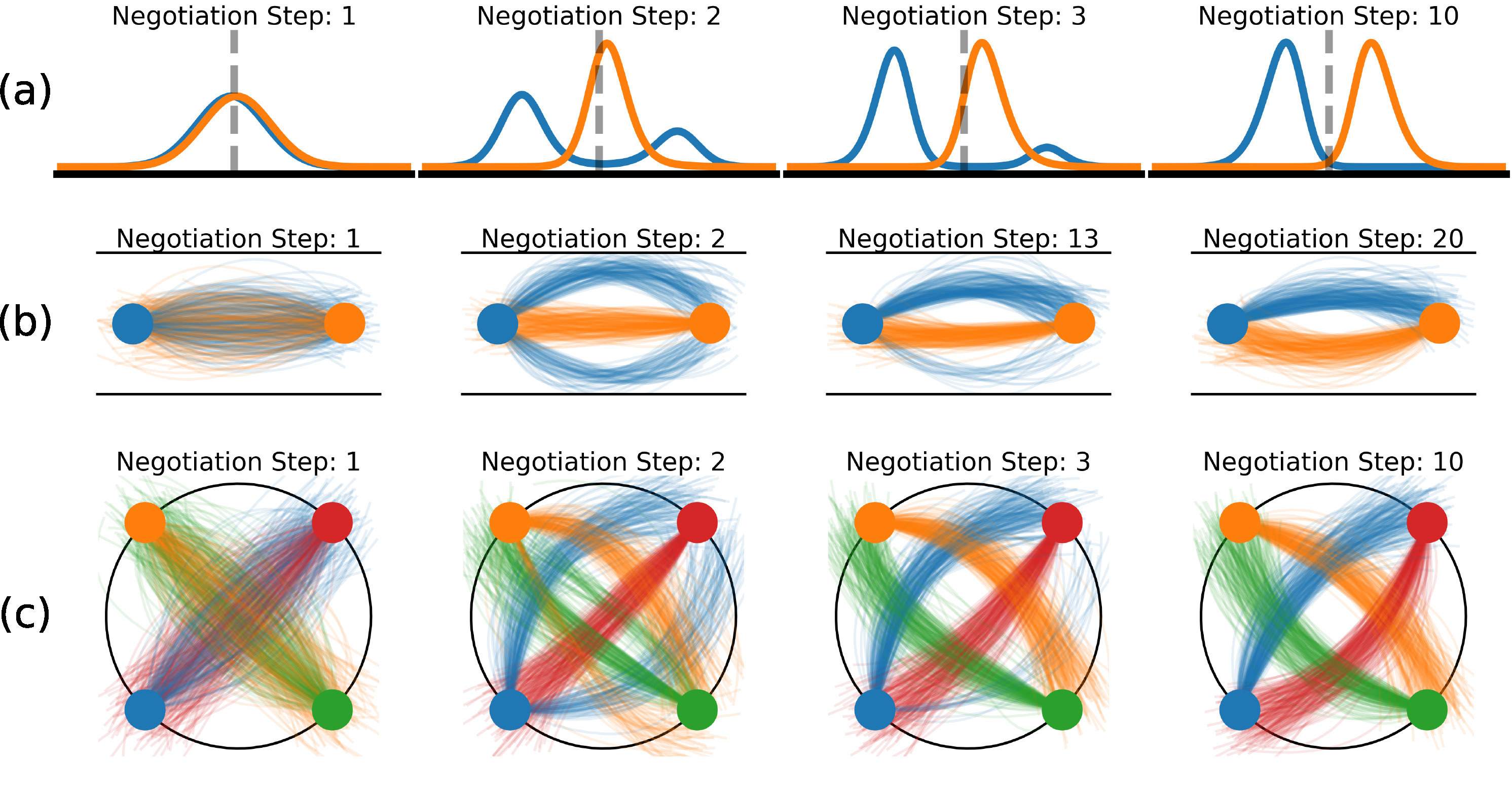}
    \captionsetup{justification=justified}
    \vspace{-1em}
    \caption{Examples of the iterative Bayesian update process. (a) Two-agent negotiation in one-dimensional space. (b) Two-agent hallway passing, where the mixed strategy is visualized as trajectory samples. (c) Four-agent crossing.} 
    \label{fig:algo_examples}
\end{figure*}

\subsection{Iterative Bayesian update with two agents}

We describe the iterative Bayesian update scheme that finds cooperative mixed strategies between two navigation agents. Bayesian update requires two key components: a prior belief and a conditional likelihood function---the latter is often interpreted as a measurement model. Here, the prior belief is each agent's nominal mixed strategy $p^\prime(s)$. By assuming each agent follows the \emph{expected utility hypothesis}~\citep{von_neumann_theory_1947}, we propose the following conditional likelihood function to reflect collision avoidance behavior.

\begin{definition}{(Conditional likelihood)}
    Given the mixed strategy of agent $j$, the likelihood of measuring another agent following a pure strategy $s(t)$ is defined as:
    \begin{align}
        \color{black} z(s\vert p_j) & \color{black} \propto \exp\Big({-}\mathbb{E}_{p_j}[r](s)\Big). \label{eq:conditional_likelihood}
    \end{align}
\end{definition}

\begin{remark}
    The inverse exponential in (\ref{eq:conditional_likelihood}) is an empirical design choice, inspired by Gaussian distributions. But as will be shown later, the inverse exponential is necessary for the algorithm to converge to Nash equilibrium.
\end{remark}

\begin{definition}{(Bayesian posterior strategy)}
    Given the nominal mixed strategy $p_i^\prime(s)$ of agent $i$ and the current mixed strategy $p_j(s)$ of agent $j$, the Bayesian posterior mixed strategy of agent $i$ after applying Bayes' rule is:
    \begin{align}
        p_i(s) & = \eta \cdot p_i^\prime(s) \cdot z(s\vert p_j) \\
        & = \eta \cdot p_i^\prime(s) \cdot \exp\Big({-}\mathbb{E}_{p_j}[r](s)\Big), \label{eq:twoagent_posterior}
    \end{align} where $\eta$ is the normalization term.
\end{definition}

With the Bayesian posterior strategy defined in (\ref{eq:twoagent_posterior}), we now describe the complete iterative Bayesian update algorithm with two agents in Algorithm~\ref{algo:two_agent_update}. Figure~\ref{fig:algo_examples}(a) and Figure~\ref{fig:algo_examples}(b) show evolution of two agents' mixed strategies in an illustrative one-dimensional example and a hallway passing scenario, respectively.

\begin{algorithm} [t!]
    \caption{Two-agent iterative Bayesian update}
    \label{algo:two_agent_update}
    \begin{algorithmic}[1]
        \Procedure{TwoAgentUpdate}{$p_{i}^\prime, p_{j}^\prime$}
        \State $k \gets 0$ \Comment{$k$ is the negotiation step.}
        \State $p_{i}^{[k]} \gets p_{i}^\prime$
        \State $p_{j}^{[k]} \gets p_{j}^\prime$
        \While{convergence criterion not met}
            \State $p_{i}^{[k+1]}(s) \gets \eta \cdot p_{i}^\prime(s) \cdot z(s\vert p_{j}^{[k]})$
            \State $p_{j}^{[k+1]}(s) \gets \eta \cdot p_{j}^\prime(s) \cdot z(s\vert p_{i}^{[k+1]})$
            \State $k \gets k+1$
        \EndWhile
        \State \textbf{return} $p_{i}^{[k]}$ and $p_{j}^{[k]}$
        \EndProcedure
    \end{algorithmic}
\end{algorithm}

\begin{algorithm} [t!]
    \caption{Multi-agent iterative Bayesian update}
    \label{algo:multi_agent_update}
    \begin{algorithmic}[1] 
        \Procedure{MultiAgentUpdate}{$p^\prime_{1}, \dots, p^\prime_{N}$}
        \State $k \gets 0$ \Comment{$k$ is the negotiation step.}
        \For{$i\in[1,N]$}
            \State $p_{i}^{[k]}(s) \gets p^\prime_{i}(s)$
        \EndFor
        \While{convergence criterion not met}
            \For{$i\in[1,N]$}
                \State $p^{[k]}_{/i} \gets {\sum_{j<i}} p^{[k+1]}_j {+} \sum_{j>i} p^{[k]}_j$
                \State $p_{i}^{[k+1]} \gets \eta \cdot p^\prime_{i} \cdot z\left(s\vert p^{[k]}_{/i} \right)$ \Comment{See (\ref{eq:multiagent_conditional_likelihood})}
            \EndFor
            \State $k \gets k+1$
        \EndWhile
        \State \textbf{return} $p_{1}^{[k]}, \dots, p_{N}^{[k]}$
        \EndProcedure
    \end{algorithmic}
\end{algorithm}

\subsection{Iterative Bayesian update with multiple agents}

We now extend Algorithm \ref{algo:two_agent_update} to an arbitrary number $N$ of agents. We first extend the definition of the conditional likelihood function. 

\begin{definition}{(Conditional likelihood (multi-agent))}
    Given the mixed strategies of all agents other than $i$, the likelihood of measuring agent $i$ following a pure strategy $s(t)$ is defined as:
    \begin{align}
        \color{black} z(s\vert p_{/i}) & \color{black} \propto \exp\Big({-}\mathbb{E}_{p_{/i}}[r](s)\Big). \label{eq:multiagent_conditional_likelihood} \\
        \color{black} \mathbb{E}_{p_{/i}}[r](s) & \color{black} = \sum_{j\in\mathcal{I}/i} \mathbb{E}_{p_j}[r](s).
    \end{align}
\end{definition}

\begin{definition}{(Bayesian posterior strategy (multi-agent))}
    Given the nominal mixed strategy $p_i^\prime(s)$ of agent $i$ and the current mixed strategies of all the rest of the agents, the Bayesian posterior mixed strategy of agent $i$ after applying Bayes' rule is:
    \begin{align}
        p_i(s) & = \eta \cdot p_i^\prime(s) \cdot z(s\vert p_{/i}) \\
        & = \eta \cdot p_i^\prime(s) \cdot \exp\Big({-}\mathbb{E}_{p_{/i}}[r](s)\Big) ,\label{eq:multiagent_posterior}
    \end{align} where $\eta$ is the normalization term.
\end{definition}

The complete iterative Bayesian update algorithm for multi-agent crowd navigation is described in Algorithm~\ref{algo:multi_agent_update}. Figure~\ref{fig:algo_examples}(c) shows the evolution of multi-agent mixed strategies in a crossing scenario. We name our algorithm \emph{Bayesian Recursive Nash Equilibrium (BRNE)}.

Across all the examples in Figure~\ref{fig:algo_examples}, we want to highlight two important properties of our algorithm: (1) Bayesian update could capture non-symmetric and multimodal mixed strategies, which expands the expressiveness of our method; (2) Although the iterative Bayesian update scheme does not simultaneously update all agents' strategies in each iteration, the converged mixed strategies can still be symmetrical.

\begin{figure*}[ht!]
    \centering
    \includegraphics[width=\textwidth]{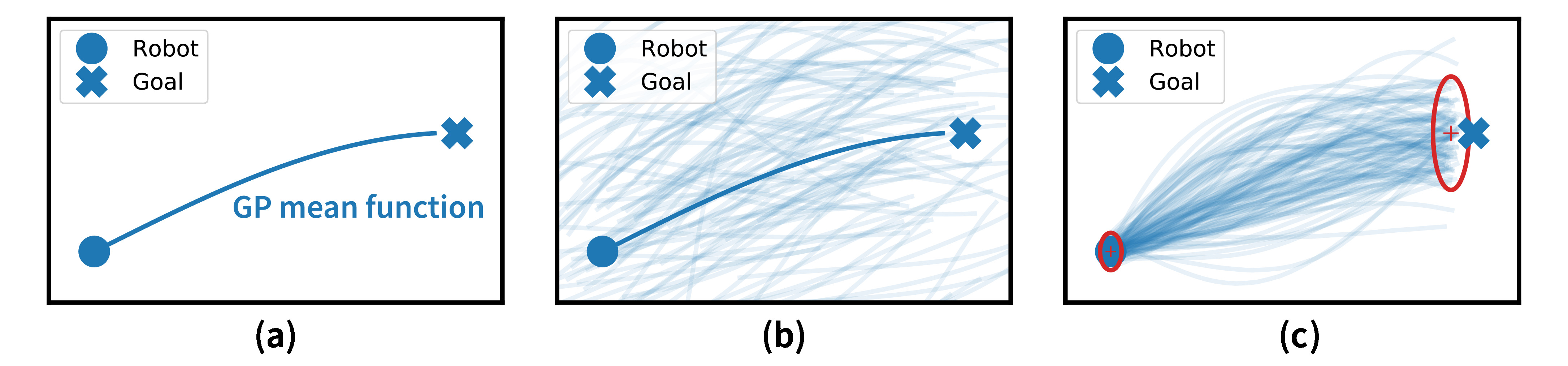}
    \captionsetup{justification=justified}
    \caption{Illustration of specifying a nominal mixed strategy with a Gaussian process (GP). (a) We first specify a trajectory as the mean function of the GP. For the robot, it would be a trajectory toward the goal, generated by a meta-planner; (b) We then specify the covariance kernel parameters, either learned from datasets or hand-tuned, which give us the GP prior distribution; (c) The GP prior is insufficient as the nominal mixed strategy. It needs to be conditioned at specific time steps with user-specified marginal uncertainty. In the figure, we condition the GP prior on the first and last time step (the specified marginal uncertainty is shown as the red ellipse); the resulting distribution is the robot's nominal mixed strategy.}
    \label{fig:gp_illustration}
    \vspace{+1em}
\end{figure*}

\begin{figure*}[ht!]
    \centering
    \includegraphics[width=\textwidth]{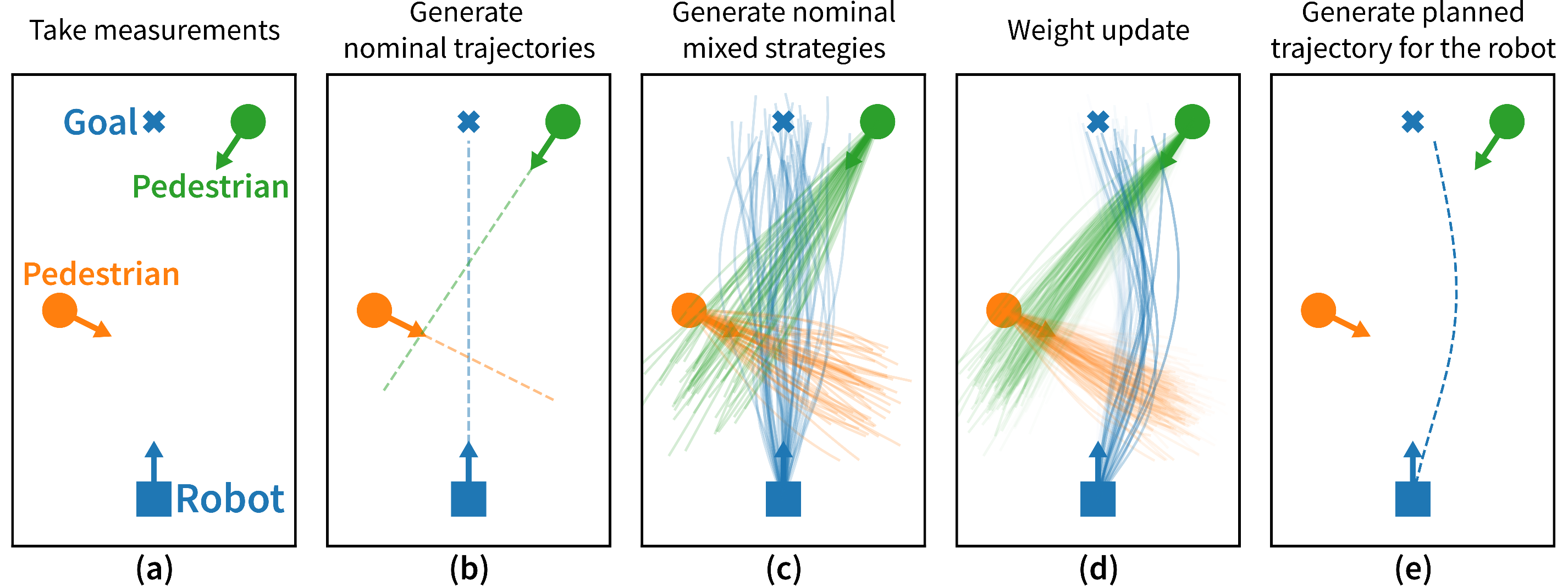}
    \captionsetup{justification=justified}
    \caption{Illustration of the model predictive crowd navigation framework. (a) The robot first takes measurements of nearby pedestrians' positions and velocities; (b) The robot then generates the mean functions of the Gaussian processes for the pedestrians and for itself; (c) The robot specifies the Gaussian processes as the nominal mixed strategies for all agents and draws trajectory samples from them; (d) The weights of the trajectory samples are updated based on Algorithm~\ref{algo:multi_agent_update} until convergence; (e) The mean of the robot's converged mixed strategy becomes the robot's planned trajectory.}
    \label{fig:algorithm_illustration}
\end{figure*}

\subsection{Iterative Bayesian update for model predictive crowd navigation}

We now apply the iterative Bayesian update scheme in Algorithm~\ref{algo:multi_agent_update} as a model predictive crowd navigation framework. The key insight is that the mixed strategy of each agent can be modeled through trajectory samples, with which Algorithm~\ref{algo:multi_agent_update} is essentially a weight update scheme for the samples, allowing fast computation in practice. Below, we introduce the practical algorithmic design. 

\subsubsection*{Motivation for sample-based representation} Computing the analytical mixed strategies following the Bayesian update in Algorithm~\ref{algo:multi_agent_update} is intractable in practice, as the mixed strategies are represented as continuous probability density functions in the analytical formula. This is a common issue among Bayesian filter methods, as computing the analytical Bayesian posterior is, in general, intractable. Importance sampling~\citep{thrun_probabilistic_2005} is one of the most widely used solutions to approximate the posterior from a Bayesian update efficiently. It represents the prior belief using samples drawn from the prior belief, then it computes the posterior belief by updating the sample weights based on the measurement model. With the sample weights updated, the samples can be updated further through resampling methods, such as rejection sampling~\citep{casella_generalized_2004}, for the next iteration. We apply importance sampling to approximate Algorithm~\ref{algo:multi_agent_update} in practice. The mixed strategies are represented as samples and weights are computed based on the Bayesian update step in Algorithm~\ref{algo:multi_agent_update}. Similar sampling-based representation is also used in stochastic optimal control, such as path integral control~\citep{theodorou_generalized_2010}.

\subsubsection*{Algorithm overview} In the proposed sampling-based model predictive crowd navigation framework, the robot repeatedly (1) observes the pedestrian position and velocity, (2) generates trajectory samples to represent the nominal mixed strategies for the pedestrians, (3) generates trajectory samples toward the navigation destination to represent the nominal mixed strategy for itself, (4) applies Algorithm~\ref{algo:multi_agent_update} to update the weights of the trajectory samples, and (5) computes the optimal control signal by tracking the weighted average trajectory from the robot's converged mixed strategy. The overall process is visualized in Figure~\ref{fig:algorithm_illustration}.

\subsubsection*{Gaussian process mixed strategy model} One of the requirements for the sampling-based Bayesian update is that we can draw samples from the prior belief, also often known as the proposal distribution, which in Algorithm~\ref{algo:multi_agent_update} is the nominal mixed strategy of each agent. We model the nominal mixed strategy for each agent using Gaussian processes~\citep{rasmussen_gaussian_2006}. A Gaussian process is an infinite-dimensional normal distribution over continuous functions. In our case, these continuous functions are the trajectories, as we model a trajectory as a continuous mapping from time to the agent state. A Gaussian process is characterized by a mean function and a covariance kernel function. We use a constant velocity model as the mean function for each pedestrian, assuming the pedestrian will keep the current velocity during the robot's planning horizon\footnote{In \citet{scholler_what_2020}, the constant velocity model is shown to be competitive against state-of-the-art learning-based models for trajectory prediction.}. A high-level meta-planner generates the mean function for the robot as a trajectory from the robot's current location toward the navigation goal. With the presence of static obstacles, the mean function trajectory should also avoid the obstacles. Such trajectory planning problems have been well-studied and can be solved using methods such as RRT or trajectory optimization. The parameters of the covariance kernel function can be learned from existing datasets through standard inference techniques such as maximum likelihood estimation~\citep{trautman_robot_2015} or hand-tuned if the number of parameters is small (e.g., with radial basis kernels). Once a Gaussian process is specified, following the same step in Gaussian process regression (Eq 2.24 in \citet{rasmussen_gaussian_2006}), we can condition the specified Gaussian process on discrete time steps. This conditioning happens offline and converts the infinite-dimensional Gaussian process to a finite-dimensional normal distribution with each dimension representing a time step. It also allows us to specify the marginal uncertainty at specific time steps. The converted multivariate normal distribution is an agent's nominal mixed strategy in practice, from which trajectory samples can be drawn efficiently during runtime, as the sampling of trajectories is now equivalent to sampling from a normal distribution. Figure~\ref{fig:gp_illustration} illustrates the steps of specifying a nominal mixed strategy with a Gaussian process. Lastly, Gaussian processes do not explicitly model agents' dynamics. Instead, the Gaussian process kernel preserves the geometrical smoothness of the trajectory samples. Similar Gaussian process-based trajectory distribution representations have been verified in other works for crowd navigation~\citep{trautman_robot_2015} and motion planning~\citep{mukadam_continuous-time_2018}.

\begin{remark}
    Even though we use a Gaussian process-based model for the mixed strategy, it is not the only choice. For pedestrians, the nominal mixed strategy can be the predicted trajectory samples drawn from neural network-based trajectory prediction frameworks, such as Trajectron~\citep{salzmann_trajectron_2020} or TrajNet~\citep{kothari_human_2022}. For the robot, instead of directly drawing samples in the trajectory space, we can also draw samples in the space of control signals by randomly perturbing a nominal control signal, similar to the widely used model predictive path integral (MPPI) control framework~\citep{theodorou_generalized_2010, williams_aggressive_2016}. We choose Gaussian processes for computation and sampling efficiency. Inference with Gaussian processes, which computes both the mean function and the covariance kernel from existing data, is known to be computationally expensive. However, we avoid performing Gaussian process inference by specifying the mean functions and the kernel parameters in advance. Therefore, drawing trajectory samples is equivalent to drawing samples from a multivariate normal distribution, which is affordable for real-time computation and faster than inference from neural networks.
\end{remark}

\subsubsection*{Weight update for samples} Once the trajectory samples are generated, we can compute the conditional likelihood function (\ref{eq:multiagent_conditional_likelihood}) for each sample as the updated weight, where the continuous integral can be approximated using Monte-Carlo integration based on the samples. After normalizing the weights, the now-weighted trajectory samples represent the Bayesian posterior in (\ref{eq:multiagent_posterior}). We repeat this process until convergence, and the weighted average trajectory from the robot's converged mixed strategy will be the robot's planned trajectory, from which the robot will compute control signals to track it. Note that, even though Algorithm~\ref{algo:multi_agent_update} requires only one agent's mixed strategy to be updated at a time, the weights of the trajectory samples from the same agent can be updated simultaneously, which can benefit from parallel computation for better computation efficiency. Pseudocode for an importance sampling-based implementation and a rejection sampling-based implementation are included in the appendix.

\subsubsection*{Algorithm complexity} We analyze the computational time complexity of a single iteration within the sampling-based algorithm. Given $T$ time steps as the planning horizon, $N$ agents, and $M$ samples for each agent, the complexity of computing the weights for one agent's mixed strategy is $O(TMN)$. For each iteration of the Bayesian update scheme, the weights need to be computed for all the agents, which leads to an overall complexity of $O(TMN^2)$ for one iteration.  Note that, given $N$ agents, $O(N^2)$ is the minimal complexity required for reasoning over all possible two-agent interaction pairs. The computation complexity of our algorithm with respect to the number of agents is lower than the state-of-the-art dynamic game solver ALGAMES~\citep{lecleach_algames_2022}, which has a computation complexity of $O(TN^3)$.

\section{Mixed strategy Nash equilibrium for crowd navigation} \label{sec:optimality_analysis}

We now show that Algorithm~\ref{algo:multi_agent_update} guarantees the recovery of the global Nash equilibrium of a mixed strategy game. Furthermore, the converged Nash equilibrium guarantees a lower-bounded reduction of the joint expected collision risk among all agents. 

\subsection{Preliminaries on game theory}

\begin{definition}{(Pure strategy game)}
    Given the strategy space $\mathcal{S}$, in a pure strategy general-sum game, each agent optimizes an individual objective function that depends on other agents' (pure) strategies:
    \begin{align}
        s_i^* = \argmin_{s_i} f_i(s_1, \dots, s_i, \dots, s_N). 
    \end{align} 
\end{definition}

Based on the above definition, it is clear that each agent cannot optimize their individual objective function in isolation from other agents. Thus, the conventional optimality criteria for single-agent optimal decision-making no longer apply. Instead, Nash equilibrium~\citep{nash_equilibrium_1950,nash_non-cooperative_1951} is proposed to describe an equilibrium state where each agent's strategy is optimal with respect to all other agents' current strategies.
\begin{definition}{(Global pure strategy Nash equilibrium)}
    A set of (pure) strategies from all agents, denoted as $(s_1^*, \dots, s_N^*)$, reach the global Nash equilibrium if and only if the following equality holds for all agents:
    \begin{align}
        s_i^* = \argmin_{s_i} f_i(s_1^*, \dots, s_i, \dots, s_N^*), \quad \forall i\in\{1,\dots,N\}. \label{eq:pure_nash_def}
    \end{align}
\end{definition}

\begin{remark}
    A Nash equilibrium is local, as opposed to global, when (\ref{eq:pure_nash_def}) only holds for a local region within the strategy space. 
\end{remark}

Pure strategy describes the \emph{deterministic} decisions of an agent. When decisions are uncertain, the game formula can be extended to mixed strategies.

\begin{definition}{(Mixed strategy game)}
    Given the strategy space $\mathcal{S}$, in a mixed strategy general-sum game, each agent optimizes an individual objective function that depends on other agents' mixed strategies:
    \begin{align}
        p_i^*(s) = \argmin_{p_i} J_i(p_1, \dots, p_i, \dots, p_N),
    \end{align} where each mixed strategy $p_i(s)$ is a probability distribution over the strategy space $\mathcal{S}$. 
\end{definition}

\begin{remark}
    In practice, the individual objective function of a mixed strategy game is often, but not necessarily, formulated as the expected value of the pure strategy objective with respect to the mixed strategies:
    \begin{align}
        J_i(p_1, \dots, p_N) = \mathbb{E}_{p_1, \dots, p_N}[f_i(s_1, \dots, s_N)].
    \end{align}
\end{remark}

Similarly, the definition of Nash equilibrium can also be extended to mixed strategies: 
\begin{definition}{(Global mixed strategy Nash equilibrium)}
    A set of mixed strategies from all agents, denoted as $(p_1^*(s), \dots, p_N^*(s))$, reach the global Nash equilibrium if and only if the following equality holds for all agents:
    \begin{align}
        p_i^*(s) = \argmin_{p_i} J_i(p_1^*, \dots, p_i, \dots, p_N^*), \quad \forall i\in\{1,\dots,N\}. \label{eq:mixed_ne_def}
    \end{align} 
\end{definition}

Based on the above definitions, we define the following $N$-player mixed strategy game for crowd navigation. The objective of player $i$ is:
\begin{align}
    \color{black} J_i(p_1,\cdots,p_{N}) & \color{black} = \mathbb{E}_{p_i,p_{/i}}[r] + D(p_i\Vert p^\prime_i) \label{eq:general_sum_player_definition} \\
    \color{black} \mathbb{E}_{p_i,p_{/i}}[r] & \color{black} = \sum_{j\in\mathcal{I}/i} \mathbb{E}_{p_i,p_j}[r]
\end{align} where $D(\cdot\Vert\cdot)$ is the Kullback-Leibler(KL)-divergence between two distributions, $\mathcal{I}$ is the set of all agent indices, and $\mathcal{I}/i$ is the set of indices excluding index $i$. We assume the nominal mixed strategies $p^\prime_i$ are given a priori and accessible by all players.

We will now show that Algorithm~\ref{algo:multi_agent_update} is guaranteed to converge to the global mixed strategy Nash equilibrium of this game.

\subsection{Mixed strategy Nash equilibrium for crowd navigation}

\begin{theorem} \label{theorem:nash_equilibrium}
    The sequence of mixed strategies $\{(p_1^{[k]},\dots,p_N^{[k]})\}_k$ in Algorithm~\ref{algo:multi_agent_update} converges to a limit point $(p_1^*,\dots,p_N^*)$ such that: 
    \begin{align}
        & \color{black} \forall \epsilon>0, \exists K\in\mathbb{N}, \nonumber \\
        & \color{black} \text{s.t. } \Big\vert F\left(p_1^{[k]},\dots,p_N^{[k]}\right)  - F\left(p_1^*,\dots,p_N^*\right) \Big\vert < \epsilon, \forall k>K, \label{eq:convergence_measure} \\
        & \color{black} F(p_1,\dots,p_N) = \sum_{i=1}^{N}\sum_{j=i+1}^{N} \mathbb{E}_{p_i,p_j}[r] + \sum_{i=1}^{N} D(p_{i}\Vert p^\prime_{i}).
    \end{align} The limit point is the global Nash equilibrium (\ref{eq:mixed_ne_def}) of the mixed strategy game (\ref{eq:general_sum_player_definition}).
\end{theorem}
\begin{proof}
    We leave details of the full proof in the appendix. The proof depends on the following theorem, which reveals another important property of the Bayesian update scheme.
\end{proof}

\begin{theorem} \label{theorem:twoagent_posterior_optimality}
    The Bayesian posterior in the two-agent iterative Bayesian update process (\ref{eq:twoagent_posterior}) is the global minimum of the following optimization problem:
    \begin{align}
        \eta \cdot p_i^\prime(s) \cdot \exp&\Big({-}\mathbb{E}_{p_j}[r](s)\Big) \nonumber \\
        & = \argmin_{p} \mathbb{E}_{p,p_j}[r] + D(p\Vert p_i^\prime) \\
        s.t. & \int_{\mathcal{S}} p(s) ds = 1 . \label{eq:equation_constraint}
    \end{align} 
\end{theorem}
\begin{proof}
    We can expand the objective function:
    \begin{align}
        & \mathbb{E}_{p,p_j}[r] + D(p\Vert p^\prime_i) \\
        & = \int_{\mathcal{S}} \mathbb{E}_{p_j}[r](s) p(s) ds + \int_{\mathcal{S}} p(s) log\left( \frac{p(s)}{p^\prime_i(s)} \right) ds .\label{eq:twoagent_bayesian_objective1}
    \end{align}
    This optimization problem can be solved as an isoperimetric problem with a subsidiary condition, the Lagrangian of which is:
    \begin{align}
        & \mathcal{L}(p, \lambda) \nonumber \\
        & = D(p\Vert p^\prime_i) {+} \int_{\mathcal{S}} p(s) \mathbb{E}_{p_j}[r](s) ds {+} \lambda \left(\int_{\mathcal{S}} p(s) ds {-} 1\right) \nonumber,
    \end{align} where $\lambda\in\mathbb{R}$ is the Lagrange multiplier. Based on the Lagrangian, the necessary condition for $p^*(s)$ to be an extreme is (Theorem 1, Page 43~\citep{gelfand_calculus_2000}):
    \begin{align}
        \frac{\partial \mathcal{L}}{\partial p}(p^*,\lambda)  & = \log p^*(s) + 1 - \log p^\prime_i(s) + \mathbb{E}_{p_j}[r](s) - \lambda \nonumber \\
        & = 0 \nonumber \\
        p^*(s) & = p^\prime_i(s) \cdot \exp\Big({-}\mathbb{E}_{p_j}[r](s) + \lambda - 1\Big) \nonumber \\
        & = p^\prime_i(s) \cdot \exp\Big({-}\mathbb{E}_{p_j}[r](s)\Big) \cdot \exp(\lambda-1). \label{eq:ibr_proof_temp1}
    \end{align} By substituting (\ref{eq:ibr_proof_temp1}) to the equality constraint (\ref{eq:equation_constraint}), we have:
    \begin{align}
        \int_{\mathcal{S}} p^*(s) ds & = \int_{\mathcal{S}} p^\prime_i(s) \exp\Big({-}\mathbb{E}_{p_j}[r](s) + \lambda - 1\Big) ds \nonumber \\
        & = \exp\left(\lambda - 1\right) \int_{\mathcal{S}} p^\prime_i(s)\exp\Big(-\mathbb{E}_{p_j}[r](s)\Big)ds \nonumber \\
        & = 1 \nonumber \\
        \exp\left(\lambda - 1\right) & = \frac{1}{\int_{\mathcal{S}} p^\prime_i(s) \exp\left(-\mathbb{E}_{p_j}[r](s)\right) ds}. \label{eq:ibr_proof_temp2}
    \end{align} Substituting (\ref{eq:ibr_proof_temp2}) back to (\ref{eq:ibr_proof_temp1}) gives us:
    \begin{align}
        p^*(s) & = \frac{p^\prime_i(s) \exp\left(-\mathbb{E}_{p_j}[r](s)\right)}{\int_{\mathcal{S}}p^\prime_i(s) \exp\left(-\mathbb{E}_{p_j}[r](s)\right)ds} \\
        & = \eta \cdot p^\prime_i(s) \cdot z(s\vert p_j).
    \end{align} Since the objective function (\ref{eq:twoagent_bayesian_objective1}) is unbounded from above, this extremum $p^*(s)$ is the global minimum of the problem. \qed
\end{proof}

\begin{remark}
    The exponential-based weight update scheme, as well as its mathematical optimality, are not unique to crowd navigation and are not restricted to the formula in this paper. For example, in \citet{hoeven_many_2018} the same weight update scheme is discussed in the context of online learning.
\end{remark}

Theorem~\ref{theorem:twoagent_posterior_optimality} can be extended to the multi-agent Bayesian update formula (\ref{eq:multiagent_posterior}).

\begin{theorem} \label{theorem:multiagent_posterior_optimality}
    The Bayesian posterior of the multi-agent Bayesian update formula (\ref{eq:multiagent_posterior}) is the global minimum of the following optimization problem:
    \begin{align}
        \eta \cdot p_i^\prime(s) \cdot \exp&\Big({-}\mathbb{E}_{p_{/i}}[r](s)\Big) \nonumber \\
        & = \argmin_{p} \mathbb{E}_{p,p_{/i}}[r] + D(p\Vert p_i^\prime) \\
        s.t. & \int_{\mathcal{S}} p(s) ds = 1 .
    \end{align} 
\end{theorem}
\begin{proof}
    Follow the same proof of Theorem~\ref{theorem:twoagent_posterior_optimality}, with the function $\mathbb{E}_{p_{j}}[r](s)$ being replaced by $\mathbb{E}_{p_{/i}}[r](s)$. \qed
\end{proof}

\begin{remark}
    The individual objective function in our game formula (\ref{eq:general_sum_player_definition}) is the summation of two sub-objectives representing the two most important behavioral properties in crowd navigation: collision avoidance and goal-orientated behavior. The first sub-objective $\mathbb{E}_{p,p_{/i}}[r]$ evaluates the expected risk of collision between agents. On the other hand, minimizing the second objective---the KL-divergence with respect to the nominal mixed strategy---preserves the agent's nominal intent during collision avoidance.
\end{remark}

\begin{remark}
    Theorem~\ref{theorem:multiagent_posterior_optimality} indicates that each Bayesian posterior (\ref{eq:multiagent_posterior}) is the agent's optimal response to the rest of the agents' current mixed strategies. This means Algorithm~\ref{algo:multi_agent_update} is equivalent to the commonly used iterative best response (IBR) scheme in game theory. However, IBR does not guarantee convergence, nor does it recover the global Nash equilibrium, while our iterative scheme is guaranteed to recover a global mixed strategy Nash equilibrium. 
\end{remark}

In addition to the guaranteed convergence, we show that the converged mixed strategies also guarantee a lower-bounded reduction in the expected collision risk.

\begin{theorem} \label{theorem:risk_reduction}
    The converged mixed strategies $(p_1^*,\dots,p_N^*)$ from Algorithm~\ref{algo:multi_agent_update} has a lower joint expected collision risk, compared to the nominal mixed strategies. Furthermore, the reduction in the expected collision risk is lower-bounded by the summation of KL-divergence between each agent's converged mixed strategy and the nominal mixed strategy:
    \begin{align}
        & \sum_{i=1}^{N}\sum_{j=i+1}^{N} \mathbb{E}_{p_i^\prime, p_j^\prime}[r] - \sum_{i=1}^{N}\sum_{j=i+1}^{N} \mathbb{E}_{p_i^*, p_j^*}[r] \geq \sum_{i=1}^{N} D(p_i^*\Vert p_i^\prime).
    \end{align}
\end{theorem}
\begin{proof}
    See appendix. 
\end{proof}

\begin{remark}
    Our model is highly generalizable for crowd navigation. Algorithm~\ref{algo:multi_agent_update}, Theorem~\ref{theorem:nash_equilibrium} and Theorem~\ref{theorem:risk_reduction} do not depend on any assumption on the specific form of mixed strategy and collision risk function, thus are compatible with arbitrary types of mixed strategies and collision risk function. On the other hand, the game objective structure (\ref{eq:general_sum_player_definition}) is specifically for crowd navigation---minimizing the KL-divergence between the converged mixed strategy and the nominal strategy does not necessarily make sense for other game-theoretic interactions. 
\end{remark}

\subsection{Connections to stochastic control and stochastic game-theoretic methods}

Our framework consists of elements commonly found in other stochastic control and stochastic game-theoretic methods. Here, we discuss the connection between our framework and representative works from these fields. 

\subsubsection*{Information-theoretic duality} The derivation of Theorem~\ref{theorem:twoagent_posterior_optimality} and Theorem~\ref{theorem:multiagent_posterior_optimality} is equivalent to the derivation of the information-theoretic duality between free energy and relative entropy (KL-divergence) in stochastic control~\citep{theodorou_relative_2012}. In particular, if we specify the mixed strategy as a stochastic control process and specify the nominal mixed strategy as the path distribution of an \emph{uncontrolled} system, then the individual objective function (\ref{eq:general_sum_player_definition}) is the upper bound of the agent’s free energy. In this case, Theorem~\ref{theorem:twoagent_posterior_optimality} and Theorem~\ref{theorem:multiagent_posterior_optimality} prove that each BRNE iteration globally minimizes the upper bound of each agent's free energy, connecting our theoretical results to other results in~\citet{theodorou_relative_2012}.

\subsubsection*{Iterative best response and importance sampling} Our BRNE algorithm is based on two fundamental numerical techniques: iterative best response (IBR) and importance sampling. IBR is a widely used technique to efficiently approximate the Nash equilibrium of a game. Importance sampling is also widely used in stochastic control, particularly path integral control frameworks~\citep{theodorou_generalized_2010}, to find the optimal control policy. In~\citet{williams_best_2018}, a model predictive path integral control framework is combined with IBR (BR-MPPI) to find the optimal stochastic control policy for a vehicle interacting with other vehicles. While our BRNE framework and BR-MPPI share similar numerical techniques, there are two fundamental differences: (1) BR-MPPI is a general-purpose game-theoretic planner and BRNE is a game-theoretic \emph{crowd navigation} planner; (2) Benefitting from its narrow scope, BRNE guarantees the convergence to the global Nash equilibrium and has sufficient computation efficiency to interact with a relatively large number of agents in crowds, while BR-MPPI does not have these advantages when applied to crowd navigation.

\subsubsection*{Bayesian update and multimodality} The stochastic formulation of BRNE is motivated by the uncertain nature of human behavior---one main advantage of the stochastic formula is that BRNE could capture the multimodality in an agent's strategy (see Figure~\ref{fig:algo_examples} for examples). Similarly, a stochastic game-theoretic planner (MPOGames) is proposed in~\citet{so_mpogames_2023}, which explicitly computes multi-modal control policies in multi-agent interaction. The key difference between BRNE and MPOGames is the assumed origin of the multimodality in agent intents. The multimodality in MPOGames comes from the assumption that non-ego agents' objectives are only partially observable to the ego agent. As a result, the ego agent maintains a belief over possible non-ego agent objectives using the Bayesian update, which creates multimodality. On the other hand, in BRNE, all agents' objectives are assumed to be shared among all agents. The multimodality in BRNE comes from the global Nash equilibrium of a mixed strategy game recovered by BRNE, where the recovered Nash equilibrium could contain multimodal mixed strategies. Essentially, MPOGames and BRNE are making orthogonal contributions: MPOGames paves the way for extending BRNE to scenarios with incomplete information regarding agent objectives, and BRNE offers a potential solution to extend MPOGames beyond Gaussian-based policies and local approximation of Nash equilibrium.

\section{Evaluation} \label{sec:evaluation}

\begin{figure*}[ht!]
    \centering
    \includegraphics[width=\textwidth]{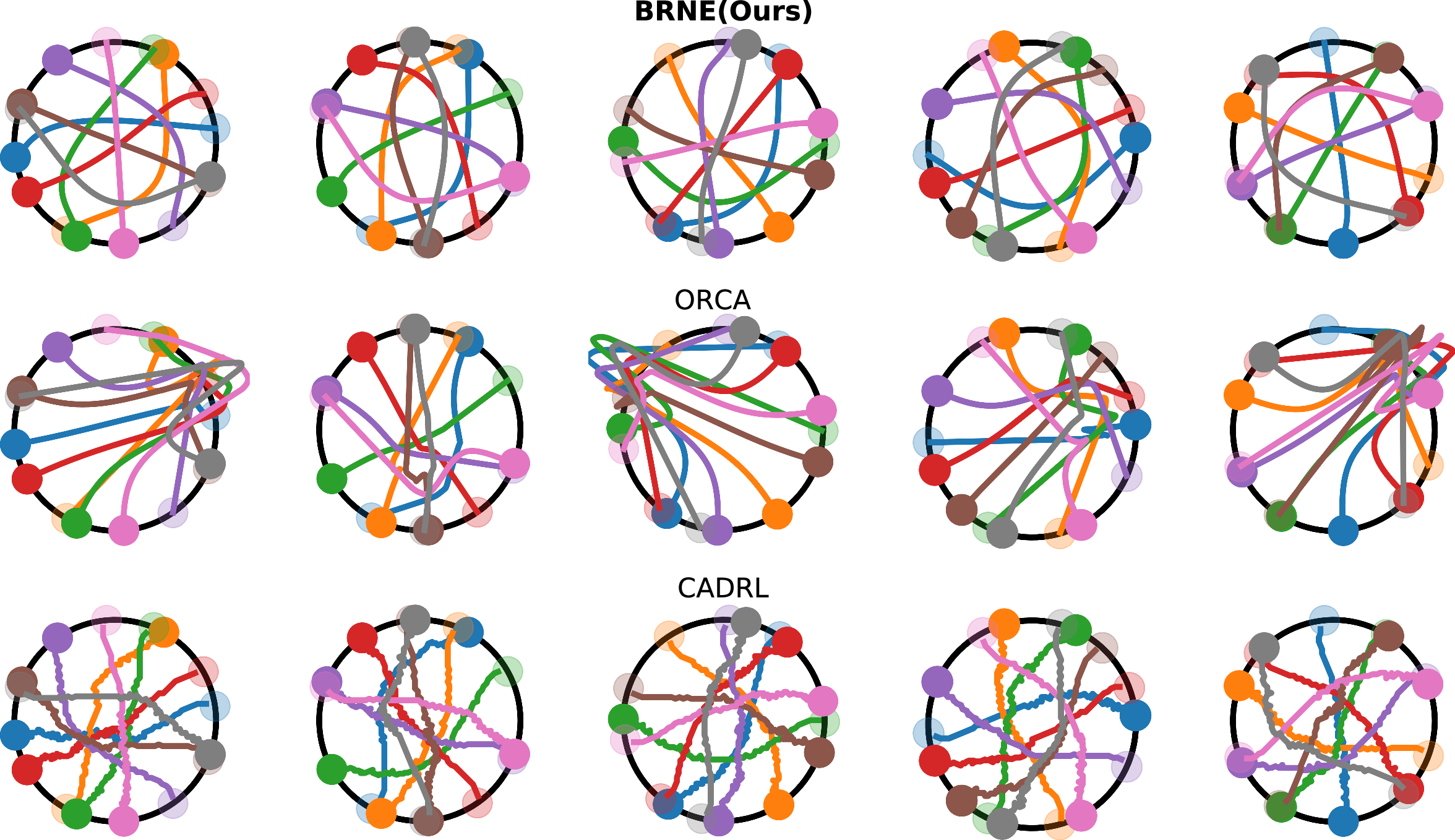}
    \captionsetup{justification=justified}
    \caption{Examples of the joint navigation strategies (8 agents) from different methods, in the multi-agent navigation experiments. }
    \label{fig:circle_example_comparison}
    \vspace{+1em}
\end{figure*}

\begin{figure*}[ht!]
    \centering
    \includegraphics[width=\textwidth]{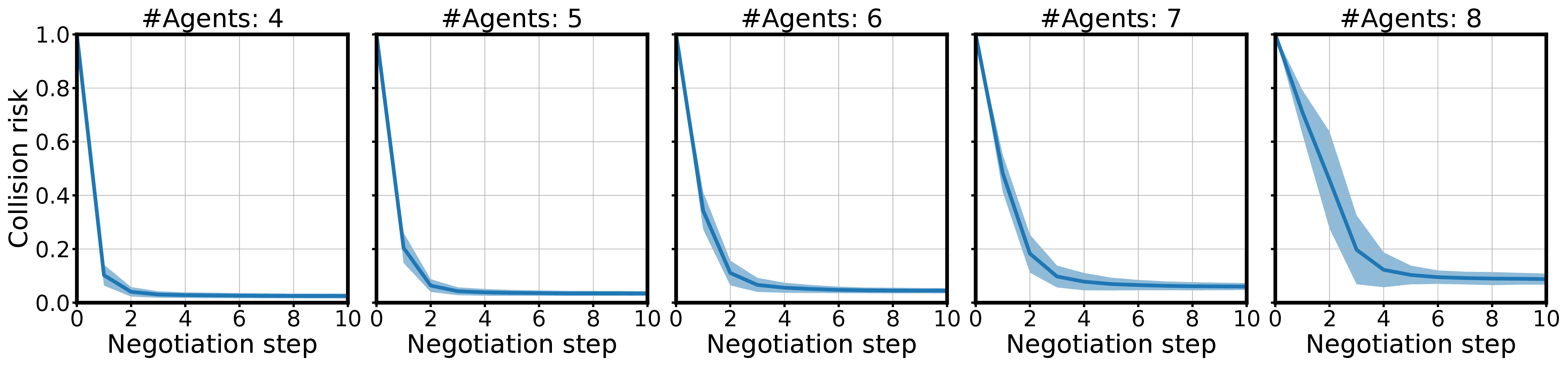}
    \captionsetup{justification=justified}
    \caption{Convergence of mixed strategy in our algorithm across different numbers of agents in the multi-agent navigation experiments.}
    \label{fig:brne_convergence}
    \vspace{+1em}
\end{figure*}

\begin{table*}[ht!] 
    \captionsetup{justification=centering}
    \caption{Safety metrics of multi-agent navigation experiments.}
    \setlength{\tabcolsep}{4.0pt}
    \begin{subtable}[t]{0.11\textwidth}
        \centering
        \begin{tabular}{|c|}
        \toprule
        \text{\#Agents} \\
        \midrule
        \textsf{BRNE} (RS) \\
        \textsf{BRNE} (IS) \\
        \textsf{ORCA} \\
        \textsf{CADRL} \\
        \bottomrule
        \end{tabular}
        \captionsetup{justification=centering}
        \caption*{}
    \end{subtable}
    \setlength{\tabcolsep}{3.0pt}
    \begin{subtable}[t]{0.32\textwidth}
        \centering
        \begin{tabular}{|ccccc|}
        \toprule
        4 & 5 & 6 & 7 & 8 \\
        \midrule
        $1\%$ & $1\%$ & $2\%$ & $4\%$ & $4\%$ \\
        $2\%$ & $3\%$ & $4\%$ & $5\%$ & $7\%$ \\
        $\bf{0\%}$ & $\bf{0\%}$ & $\bf{0\%}$ & $\bf{0\%}$ & $\bf{0\%}$ \\
        $100\%$ & $100\%$ & $100\%$ & $100\%$ & $100\%$ \\
        \bottomrule
        \end{tabular}
        \captionsetup{justification=centering}
        \caption{Collision rate ($\%$)---lower is better}
    \end{subtable}
    \setlength{\tabcolsep}{3.0pt}
    \begin{subtable}[t]{0.52\textwidth}
        \centering
        \begin{tabular}{|ccccc|}
        \toprule
        4 & 5 & 6 & 7 & 8 \\
        \midrule
        $1.20{\pm}0.18$ & $\bf{1.08{\pm}0.15}$ & $\bf{1.00{\pm}0.13}$ & $\bf{0.95{\pm}0.15}$ & $\bf{0.93{\pm}0.14}$ \\
        $\bf{1.24{\pm}0.23}$ & $1.07{\pm}0.18$ & $0.96{\pm}0.18$ & $0.87{\pm}0.15$ & $0.76{\pm}0.13$ \\
        $0.6{\pm}0.0$ & $0.6{\pm}0.0$ & $0.6{\pm}0.0$ & $0.6{\pm}0.0$ & $0.6{\pm}0.0$ \\
        $0.17{\pm}0.15$ & $0.08{\pm}0.07$ & $0.05{\pm}0.04$ & $0.04{\pm}0.03$ & $0.04{\pm}0.02$ \\
        \bottomrule
        \end{tabular}
        \captionsetup{justification=centering}
        \caption{Safety distance (m)---larger is better}
    \end{subtable}
    \label{table:multinav_safety}
\end{table*}

\begin{table*}[ht!]
    \captionsetup{justification=centering}
    \caption{Path efficiency of multi-agent navigation experiments.}
    \centering
    \setlength{\tabcolsep}{12.0pt}
    \begin{subtable}[t]{1.0\textwidth}
        \centering
        \begin{tabular}{|c|ccccc|}
        \toprule
        \text{\#Agents} & 4 & 5 & 6 & 7 & 8 \\
        \midrule
        \textsf{BRNE} (RS) & $6.44{\pm}0.06$ & $\bf{6.61{\pm}0.09}$ & $\bf{6.76{\pm}0.08}$ & $\bf{6.84{\pm}0.08}$ & $\bf{7.00{\pm}0.07}$ \\
        \textsf{BRNE} (IS) & $6.90{\pm}0.32$ & $7.06{\pm}0.34$ & $7.23{\pm}0.33$ & $7.36{\pm}0.32$ & $7.36{\pm}0.31$ \\
        \textsf{ORCA} & $\bf{6.38{\pm}0.35}$ & $6.98{\pm}1.10$ & $7.45{\pm}1.46$ & $7.22{\pm}1.38$ & $7.14{\pm}1.42$ \\
        \textsf{CADRL} & $8.96{\pm}1.12$ & $9.02{\pm}1.18$ & $8.98{\pm}0.71$ & $9.26{\pm}1.10$ & $9.03{\pm}0.53$ \\
        \bottomrule
        \end{tabular}
        \captionsetup{justification=centering}
        \caption{Maximum path length (m)---shorter is better}
    \end{subtable} \\
    \centering
    \setlength{\tabcolsep}{12.0pt}
    \label{table:multinav_efficiency}
\end{table*}

\subsection{Overview of evaluation}
We evaluate our BRNE navigation method, as well as other non-learning and learning-based crowd navigation methods, in two categories of tasks: (1) multi-agent navigation (Section~\ref{sec:multi_agent_exp}); (2) crowd navigation (Section~\ref{sec:social_nav_exp1}, Section~\ref{sec:social_nav_exp2}). In the multi-agent navigation experiment, the tested methods are asked to plan paths for all the agents, given a set of start and goal locations. The multi-agent navigation experiment represents the most ideal situation of crowd navigation, where all agents' behaviors follow the behavioral model assumed by the planning algorithm. The results of the multi-agent navigation experiment do not necessarily reflect an algorithm's crowd navigation performance with humans, but serve as an upper bound of the algorithms' ability to infer safe and efficient joint paths---if a method cannot infer safe and efficient joint paths even when all the agents follow the assumed behavioral model, then the method is unlikely to generate safe and efficient navigation actions with real humans who do not necessarily follow the assumed behavioral model. To evaluate the crowd navigation performance of the tested methods, especially when pedestrians do not follow the behavior assumed by the planning algorithm, we conduct extensive crowd navigation benchmark experiments with two different kinds of experiment designs. The first crowd navigation benchmark experiment (Section~\ref{sec:social_nav_exp1}) places all the agents in a tight space where collision avoidance is necessary for all the agents. The pedestrian behavior is simulated using a decentralized collision avoidance algorithm ORCA~\citep{van_den_berg_reciprocal_2011}. The purpose of this evaluation is to stress test the robot's ability to infer and adapt to cooperative but unknown joint navigation strategies from other agents. The second crowd navigation benchmark experiment (Section~\ref{sec:social_nav_exp2}) aims to benchmark the algorithms' performance in realistic crowd navigation scenarios. We evaluate the performance using a state-of-the-art crowd navigation benchmark~\citep{biswas_socnavbench_2022}, which consists of real-world environmental geometry and real human navigation trajectories collected from several datasets. The robot has to navigate in unstructured spaces while interacting with human pedestrians. For each experiment, we evaluate the algorithm performance from two categories of metrics: safety and navigation efficiency. Details of the specific metrics are provided in the corresponding subsections. Lastly, we implement the algorithm on an untethered quadruped robot with fully onboard perception and computation to demonstrate our algorithm's practicality with real humans. The implementation of the algorithm, including the parameters and tutorials, will be released on \href{https://sites.google.com/view/brne-crowdnav}{https://sites.google.com/view/brne-crowdnav}.

\subsection{Multi-agent navigation evaluation} \label{sec:multi_agent_exp}

\subsubsection*{Experiment design} For multi-agent navigation, we randomly place different numbers of agents on a circle with a radius of 3 meters. The goal of each agent is to reach the other side of the circle. This design means the most efficient path for each agent is to follow a straight line across the center of the circle, which makes it necessary for all agents to compromise efficiency in exchange for safety. When randomizing the initial agent locations, we uniformly sample locations on the circle, but also make sure no two agents' initial locations are within 0.6 meters of each other---we assume each agent is circular-shaped with a radius of 0.3 meters. We vary the number of agents from 4 to 8 and conduct 100 navigation trials for each number of agents. We specify the desired velocity for all agents and all baselines as $1.2 m/s$. We test our method \textsf{BRNE} with both an importance-sampling implementation and a rejection-sampling implementation, we denote the rejection-sampling implementation as \textsf{BRNE}(RS) and the importance-sampling implementation as \textsf{BRNE}(IS).

\subsubsection*{Rationale of baselines}
We select two baselines for comparison: (1) Optimal reciprocal collision avoidance (\textsf{ORCA})~\citep{van_den_berg_reciprocal_2011}; and (2) collision avoidance with deep reinforcement learning (\textsf{CADRL})~\citep{chen_socially_2017}. We choose the selected baselines partially based on their widely accessible implementations. Furthermore, we choose \textsf{ORCA} as it is the de facto non-learning decentralized collision avoidance model for simulating crowds. We choose \textsf{CADRL} as it is the de facto learning-based crowd navigation model.  

\subsubsection*{Rationale of metrics}
For safety, we measure \emph{safety distance} and \emph{collision rate}. The safety distance is defined as the closest distance between any two agents during the whole navigation task. We consider a navigation trial to involve a collision if the safety distance of the trial is below 0.6 meters, twice the body radius of each agent. For the navigation efficiency, we measure the maximum path length among all agents in each trial. 

\begin{figure}[t!]
    \centering    \includegraphics[width=0.49\textwidth]{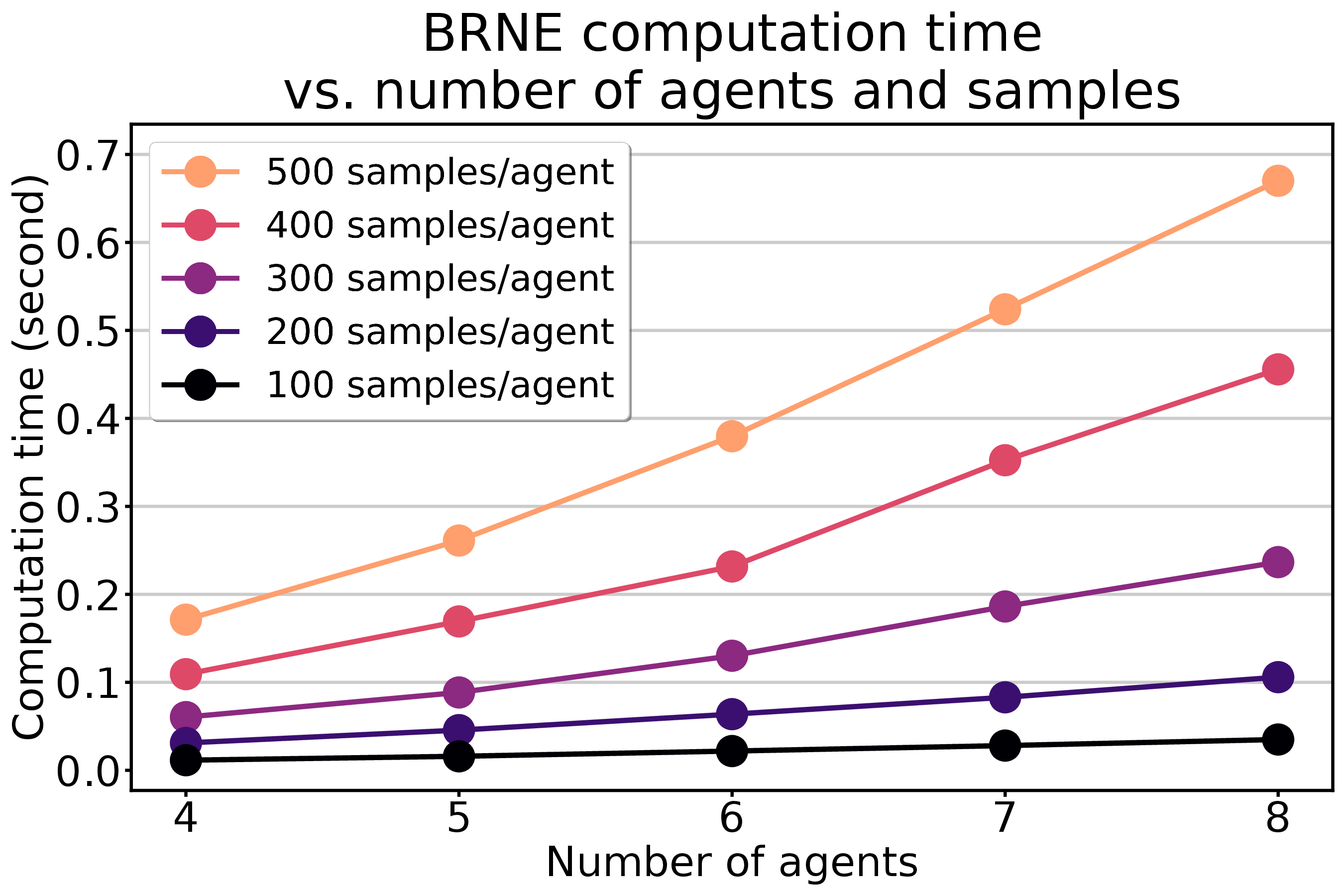}
    \captionsetup{justification=justified}
    \vspace{-2em}
    \caption{Plot of \textsf{BRNE} replanning frequencies with different numbers of agents and different numbers of samples per agent.}
    \label{fig:time_plot}
\end{figure}

\begin{table}[t!]
    \centering
    \captionsetup{justification=centering}
    \caption{Computation time of game-theoretic methods.}
    \setlength{\tabcolsep}{4.0pt}
    \begin{tabular}{|c|ccc|}
        \toprule
        Algorithm & 2 Agents & 3 Agents & 4 Agents \\
        \midrule
        \textsf{BRNE}(Ours) & $3.9{\pm} 0.1$ ms & $7.1{\pm} 0.2$ ms & $11.0{\pm}0.6$ ms \\
        \textsf{ALGAMES} & $50{\pm}11$ ms & $116{\pm}22$ ms & $509{\pm}33$ ms \\
        \textsf{iLQGames} & $752{\pm}168$ ms & $362{\pm}93$ ms & $1905{\pm}498$ ms \\
        \bottomrule
    \end{tabular}
    \label{table:speed_comparison}
\end{table}

\subsubsection*{Results} The evaluation results are presented in Table~\ref{table:multinav_safety} (safety) and Table~\ref{table:multinav_efficiency} (navigation efficiency). Note that \textsf{ORCA} guarantees collision-free paths, thus it has $0\%$ collision rate across all experiments. On the other hand, our method \textsf{BRNE} generates joint navigation strategies with larger safety distance, better efficiency, and lower disparity. Even though \textsf{BRNE} strategies are not collision-free, the collision rate is still minimal. Since the multi-agent navigation benchmark assumes all agents follow the same behavioral model, having a $0\%$ collision rate in this benchmark does not necessarily mean collision-free with humans. More importantly, as shown in Figure~\ref{fig:circle_example_comparison}, the collision-free guarantee of \textsf{ORCA} could significantly compromise the path quality, leading to highly unnatural joint strategies when agents are unevenly distributed. This indicates that, when used as a coupled prediction and planning framework on the robot, \textsf{ORCA} could plan the robot's path based on highly unrealistic expectations of human reactions. Lastly, the learning-based baseline \textsf{CADRL} performs poorly across all metrics and exhibits highly non-smooth trajectories as shown in Figure~\ref{fig:circle_example_comparison}. Note that the \textsf{CADRL} policy is trained using the toolbox provided in~\cite {chen_crowd-robot_2019} with pedestrians simulated with \textsf{ORCA}. Thus, the trained \textsf{CADRL} policy is a collision avoidance policy in crowds. We test the \textsf{CADRL} policy in this multi-agent navigation benchmark to verify if the policy is a coupled prediction and planning method, instead of a dynamic object avoidance method. 
Videos of the multi-agent navigation experiments are included in our project website: \href{https://sites.google.com/view/brne-crowdnav}{https://sites.google.com/view/brne-crowdnav}.

\begin{table*}[ht!]
    \captionsetup{justification=centering}
    \caption{Results of simulated crowd navigation experiments.}
    \setlength{\tabcolsep}{8.0pt}
    \centering
    \begin{tabular}{|c|cc|cc|}
        \toprule
        \text{Algorithm} & \text{Collision rate ($\%$)} & \text{Safety distance (m)} & \text{Time to goal (s)} & \text{Path length (m)} \\
        \midrule
        \textsf{BRNE}(Ours) & $18\%$ & $\bf{0.78{\pm}0.18}$ & $8.29\pm 0.39$ & $1.10\pm 0.09$ \\
        \textsf{NCE} & $15\%$ & $0.70\pm 0.08$ & $\bf{8.12\pm 1.91}$ & $1.15\pm 0.14$ \\
        \textsf{ORCA} & $\bf{0\%}$ & $0.6\pm 0.0$ & $10.13\pm 0.70$ & $\bf{1.02\pm 0.02}$ \\
        \textsf{CADRL} & $96\%$ & $0.41\pm 0.12$ & $24.10\pm 0.00$ & $1.29\pm 0.16$\\
        \bottomrule
    \end{tabular}
    \captionsetup{justification=centering}
    \label{table:simulated_crowd_navigation}
\end{table*}

\begin{figure*}[ht!]
    \centering
    \includegraphics[width=\textwidth]{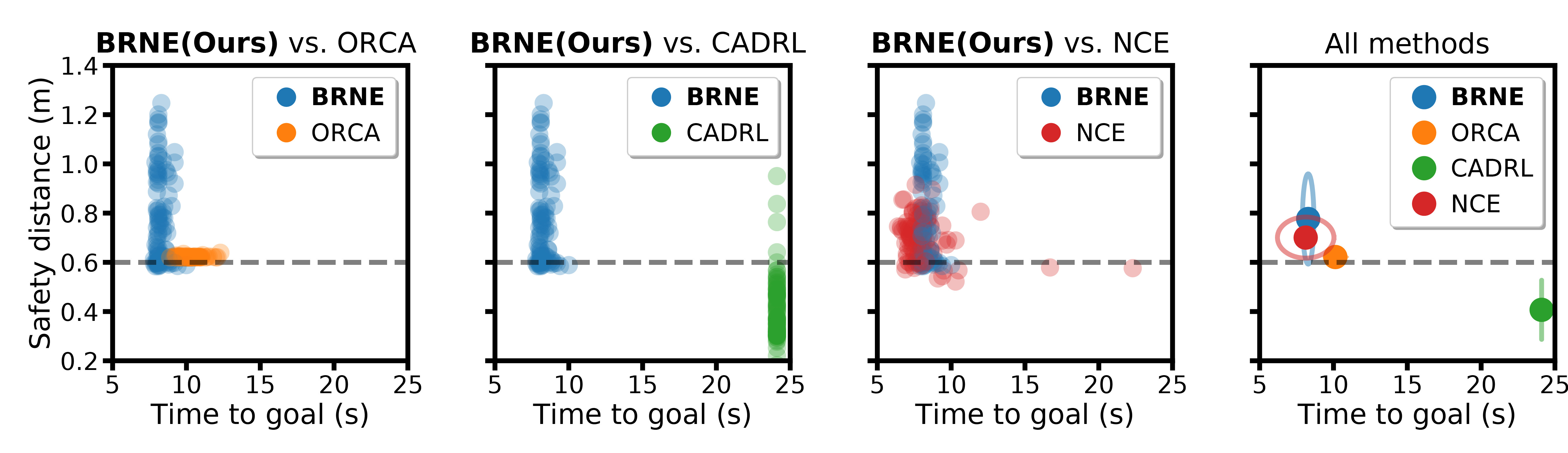}
    \captionsetup{justification=justified}
    \vspace{-2em}
    \caption{Plots of time-to-goal metric vs. safety distance metric for all the methods. From left to right, the first three figures compare our method \textsf{BRNE} and other methods by overlapping data points from all 100 trials. The fourth figure overlaps the statistics (mean $\pm$ standard deviation) of all methods.}
    \label{fig:simulated_result_plots}
\end{figure*}

\begin{figure*}[ht!]
    \centering
    \includegraphics[width=\textwidth]{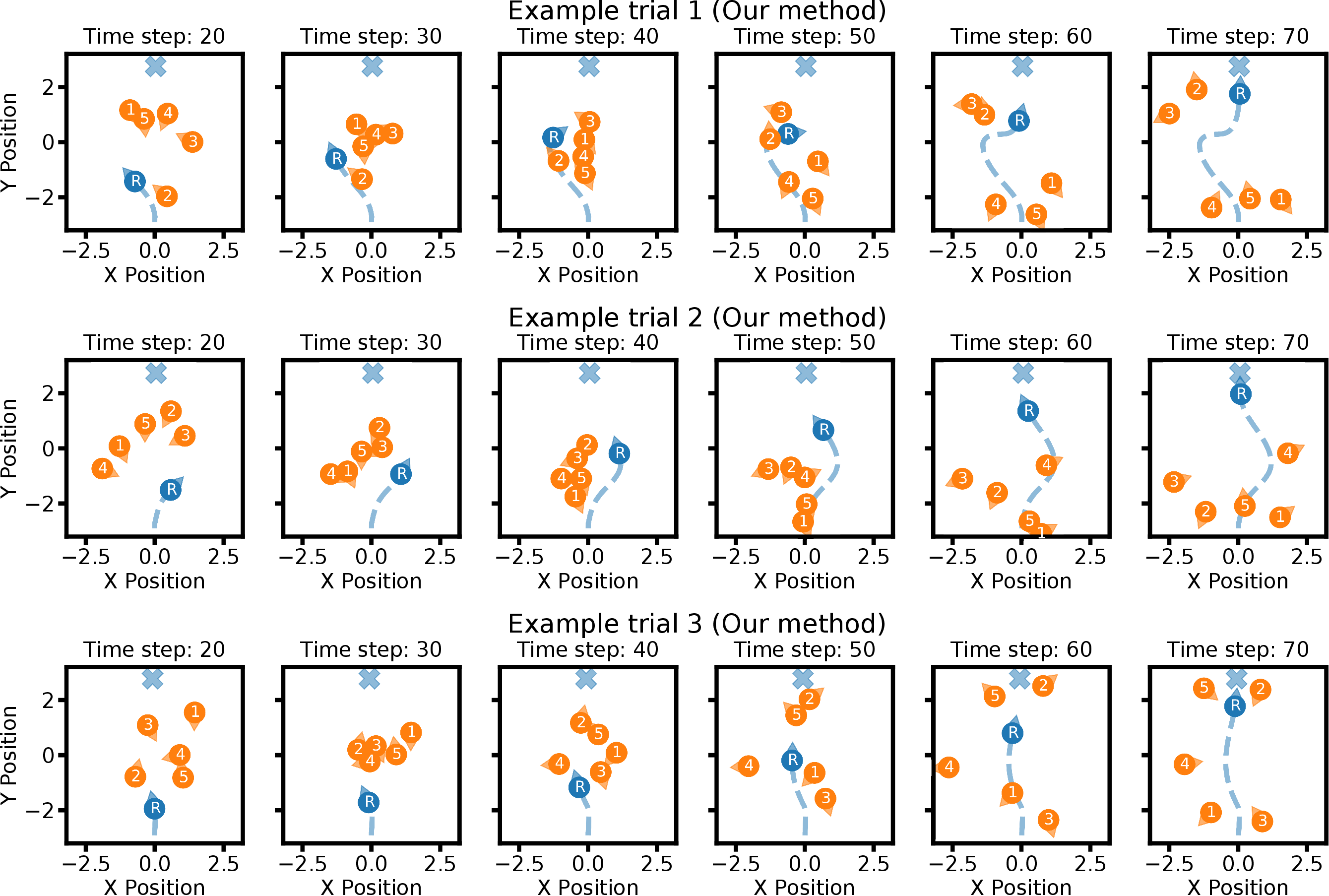}
    \captionsetup{justification=justified}
    \caption{Frames from three example trials of our method \textsf{BRNE} in the simulated crowd navigation experiments. The blue circle with the letter ``R'' represents the robot and the orange circles with numbers represent the five simulated pedestrians. The robot's task is to navigate to the goal indicated by the blue cross without running into pedestrians.}
    \label{fig:simulated_frame_plots}
\end{figure*}

\subsubsection*{Computation time of BRNE} We use the multi-agent navigation benchmark to evaluate how the number of agents and samples affect the computation time of \textsf{BRNE}. We fixed the planning horizon of \textsf{BRNE} at 50 time steps (5 seconds), then we vary the number of agents from 4 to 8 and the number of samples per agent from 100 to 500. The tested \textsf{BRNE} implementation uses Numba for parallelization and runs on a laptop with Intel Core i9-12900K Processor (20 threads). The results are shown in Figure~\ref{fig:time_plot}. We can see that when the algorithm is accelerated through parallelization in practice, the computation complexity is close to linear with respect to the number of agents. Since the scalability with the number of agents is often a computation bottleneck for game-theoretic planners, in Table~\ref{table:speed_comparison}, we further compare the computation speed of \textsf{BRNE} with two state-of-the-art dynamic game solvers, \textsf{ALGAMES}~\citep{lecleach_algames_2022} and \textsf{iLQGames}~\citep{fridovich-keil_efficient_2020}. We use 100 samples per agent for \textsf{BRNE} and use the results reported in the ``intersection'' experiment from~\citet{lecleach_algames_2022}, as it is the most similar to the crowd navigation setting in this test. We use the importance-sampling-based \textsf{BRNE} implementation throughout this test. In Table~\ref{table:speed_comparison}, we can see that BRNE is at least one order of magnitude faster than \textsf{ALGAMES} and \textsf{iLQGames} and is the only method with sufficient computation time for real-time planning with 4 agents. \textsf{ALGAMES} and \textsf{iLQGames} are general-purpose game-theoretic planners that can be applied to applications other than crowd navigation, while \textsf{BRNE} is designed specifically for crowd navigation. \textsf{BRNE} explicitly benefits from its narrow scope for better computation efficiency in dense crowds.

\subsubsection*{Convergence test} We use the multi-agent navigation benchmark to further test the convergence of the \textsf{BRNE} algorithm. In Figure~\ref{fig:brne_convergence}, we show the convergence of \textsf{BRNE} with the number of agents varying from 4 to 8. We can see the algorithm consistently converges within 10 iterations and the convergence rates are similar across different numbers of agents. We use the rejection-sampling-based implementation for the convergence test for better numerical accuracy.

\subsection{Simulated crowd navigation evaluation} \label{sec:social_nav_exp1}

\subsubsection*{Experiment design} The geometrical layout of the simulated crowd navigation experiment is similar to the multi-agent navigation experiment, where we place agents across a circle with a radius of 3 meters and each agent's task is to reach the other side of the circle. The difference in this benchmark is that we only have control over one agent (the robot), while other agents are cooperative but make independent decisions. We use \textsf{ORCA} to simulate pedestrians since it is the most commonly used decentralized collision avoidance method. We test each navigation method with $5$ simulated pedestrians. The desired velocity for each agent is set as $1.2 m/s$. The body radius of each agent is $0.3m$ and the distance threshold for a collision is $0.6m$. We randomly initialize agent locations for 100 trials, with the same randomization scheme in the multi-agent navigation experiments. In this experiment, \textsf{BRNE} is implemented with importance sampling as a model predictive planner, updating its inference of mixed strategy Nash equilibrium based on the latest observation of other agents' positions and velocities.

\begin{table*}[]
    \centering
    \captionsetup{justification=centering}
    \caption{Results of human dataset crowd navigation experiments.}
    \setlength{\tabcolsep}{10.0pt}
    \begin{tabular}{|c|cc|cc|}
        \toprule
        Algorithm & \begin{tabular}[c]{@{}c@{}}Pedestrian collisions\end{tabular} & \begin{tabular}[c]{@{}c@{}}Freezing behavior\end{tabular} & \begin{tabular}[c]{@{}c@{}}Path length (m)\end{tabular} & \begin{tabular}[c]{@{}c@{}}Time to goal (s) \end{tabular}\\
        \midrule
        \textsf{BRNE}(Ours) & $\mathbf{1}$ & $\mathbf{0}$ & $16.56\pm 3.85$ & $18.95\pm 4.81$ \\
        \textsf{Meta-planner}(Ours) & $37$ & $\mathbf{0}$ & $\mathbf{15.42\pm 3.71}$ & $17.66\pm 4.28$ \\
        \midrule
        \textsf{SF} & $\mathbf{1}$ & $1$ & $17.25\pm 4.05$ & $15.93\pm 4.17$ \\
        \textsf{ORCA} & $15$ & $1$ & $17.66\pm 5.22 $ & $22.06\pm 7.78$ \\
        \textsf{CADRL} & $40$ & $1$ & $15.70\pm 3.72$ & $\mathbf{15.14\pm 4.21}$ \\
        \textsf{Baseline} & $64$ & $1$ & $15.88\pm 3.57$ & $16.08\pm 3.73$ \\
        \bottomrule
    \end{tabular}
    \label{table:benchmark_crowd_navigation}
\end{table*}

\begin{figure*}[ht!]
    \centering
    \includegraphics[width=\textwidth]{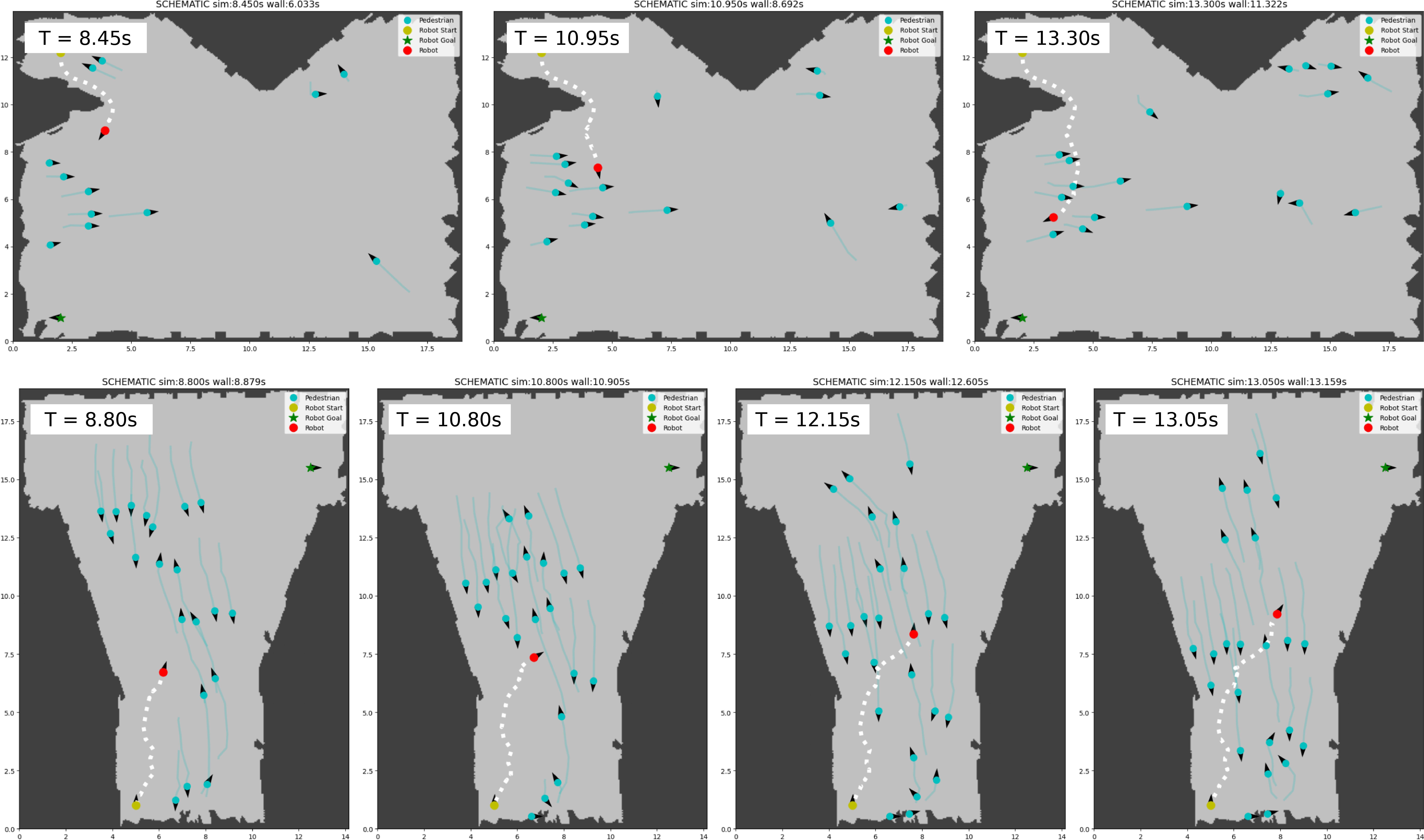}
    \captionsetup{justification=justified}
    \caption{Each row shows representative snapshots from a benchmark trial. The top row is from the UCY dataset and the bottom row is from the ETH dataset.}
    \label{fig:benchmark_frames}
\end{figure*}

\begin{figure*}[ht!]
    \centering
    \includegraphics[width=\textwidth]{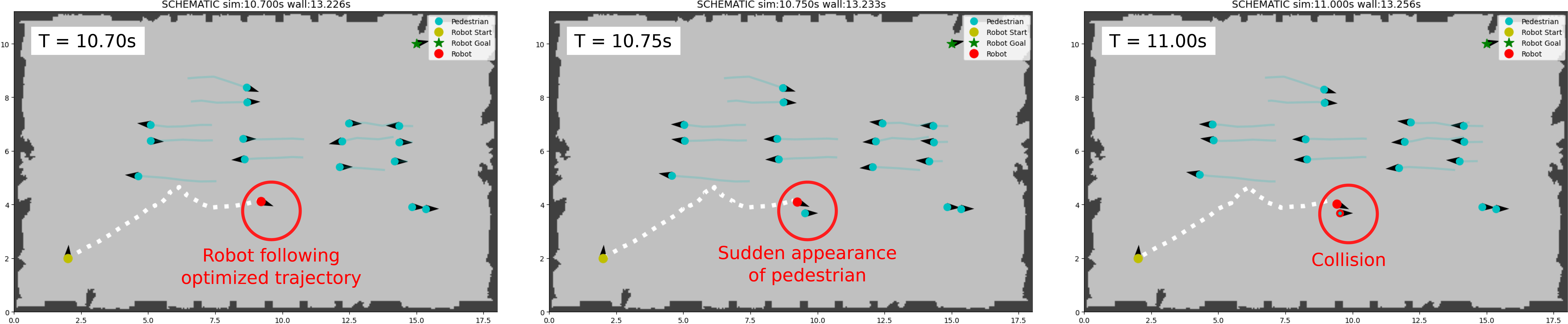}
    \captionsetup{justification=justified}
    \caption{Snapshots from the only trial with collision. The collision is caused by the sudden appearance of a pedestrian right in front of the robot, leaving the robot no space and time to react.}
    \label{fig:socnavbench_collision_trial}
    \vspace{-1em}
\end{figure*}

\subsubsection*{Rationale of baselines} We select $3$ baselines to compare with: (1) \textsf{ORCA}; (2) \textsf{CADRL}; (3) Constrastive learning social navigation (\textsf{NCE})~\citep{liu_social_2021}. We choose \textsf{ORCA} and \textsf{CADRL} for the same reason as we choose them for multi-agent navigation evaluation. We choose \textsf{NCE} since it is reported as the state-of-the-art method in simulated crowd navigation benchmarks. We train \textsf{NCE} in the environment provided by~\citet{liu_social_2021}. 

\subsubsection*{Rationale of metrics} For safety, we measure the same \emph{safety distance} and \emph{collision rate} metrics. For navigation efficiency, we measure the \emph{time-to-goal} and \emph{path length} of the robot.

\subsubsection*{Results} In Table~\ref{table:simulated_crowd_navigation} and Figure~\ref{fig:simulated_result_plots}, we show the safety and navigation efficiency results. In Figure~\ref{fig:simulated_frame_plots}, we show representative frames of our method from the simulated tests. We can see that our method \textsf{BRNE} is competitive on both navigation safety and efficiency with the state-of-the-art reinforcement learning method \textsf{NCE}. Both \textsf{NCE} and \textsf{CADRL} are tested in the same simulation environment where it is trained, while our method \textsf{BRNE} requires no training. \textsf{CADRL} performs poorly in this benchmark, exhibiting a high collision rate and low navigation efficiency. Furthermore, \textsf{CADRL} fails to reach the navigation goal and we have to terminate the trial after a certain amount of time, causing the low standard deviation in the time-to-goal metric. The performance of \textsf{CADRL} could be due to the difference between the testing environment provided by~\citet{liu_social_2021} and the training environment of \textsf{CADRL}. Videos of the simulated crowd navigation experiments are included in our project website: \href{https://sites.google.com/view/brne-crowdnav}{https://sites.google.com/view/brne-crowdnav}.

\subsection{Crowd navigation in human datasets} \label{sec:social_nav_exp2}

\subsubsection*{Experiment design} We evaluate crowd navigation with prerecorded human pedestrian behaviors within the same geometrical spaces of recordings. The evaluation is conducted in \emph{SocNavBench}~\citep{biswas_socnavbench_2022}, a state-of-the-art, high-fidelity crowd navigation benchmark framework. \emph{SocNavBench} contains human pedestrian data from two of the most commonly used pedestrian datasets ETH~\citep{pellegrini_youll_2009} and UCY~\citep{lerner_crowds_2007}. A set of 33 curated episodes are extracted from the original datasets for benchmarking, which assemble a set of highly interactive test trials. The curated episodes contain an average of 44 pedestrians per episode. The distance between the goal location and the start location for the robot ensures that the navigation task can be finished within 25 seconds if the robot does not avoid collision and navigates with the maximum permitted velocity of $1.2m/s$. An occupancy grid map of the real-world space is provided for each episode, the robot has access to the map a priori and can only navigate in the free space. Note that, even though the pedestrian agents in \emph{SocNavBench} are non-reactive given the pre-recorded human behavior, the robot agent still assumes the pedestrian agents are real-world reactive humans. We refer readers to~\citet{biswas_socnavbench_2022} for more details about the benchmark, as well as the discussion on the choice of using pre-recorded human behavior. In this experiment, \textsf{BRNE} is implemented with importance sampling as a model predictive planner, updating its inference of mixed strategy Nash equilibrium based on the latest observation of other agents' positions and velocities. The algorithm implementation for the benchmark (including paramerters) will be released on \href{https://sites.google.com/view/brne-crowdnav}{https://sites.google.com/view/brne-crowdnav}.

\subsubsection*{Rationale of baselines} We compare our method \textsf{BRNE} with the four baselines and the corresponding results reported in the original paper of \emph{SocNavBench}~\citep{biswas_socnavbench_2022}. They are (1) the social force model (\textsf{SF})~\citep{helbing_social_1995}; (2) \textsf{ORCA}; (3) \textsf{CADRL}; and (4) a pedestrian-unaware baseline provided by the benchmark with static obstacle avoidance capability (\textsf{Baseline}). Note that \textsf{ORCA} and \textsf{CADRL} do not have static obstacle avoidance capability. Thus, they use \textsf{Baseline} as a meta-planner to avoid static obstacles in the benchmark. We tried to use \textsf{Baseline} directly as the meta-planner to generate the nominal mixed strategy for \textsf{BRNE}. However, even though \textsf{Baseline} is sufficient for \textsf{BRNE} to avoid static obstacles, it is not sufficient for \textsf{BRNE} to consistently reach the goal---recall that \textsf{BRNE} requires a strong goal-oriented nominal mixed strategy. Therefore, we slightly modified \textsf{Baseline} to keep its static obstacle avoidance module but improve its goal-reaching capability. We name this modified baseline planner \textsf{Meta-Planner} and report its performance as well. We will release the implementation of both \textsf{BRNE} and \textsf{Meta-Planner} in this benchmark. 

\begin{figure*}[htbp]
    \centering
    \includegraphics[width=\textwidth]{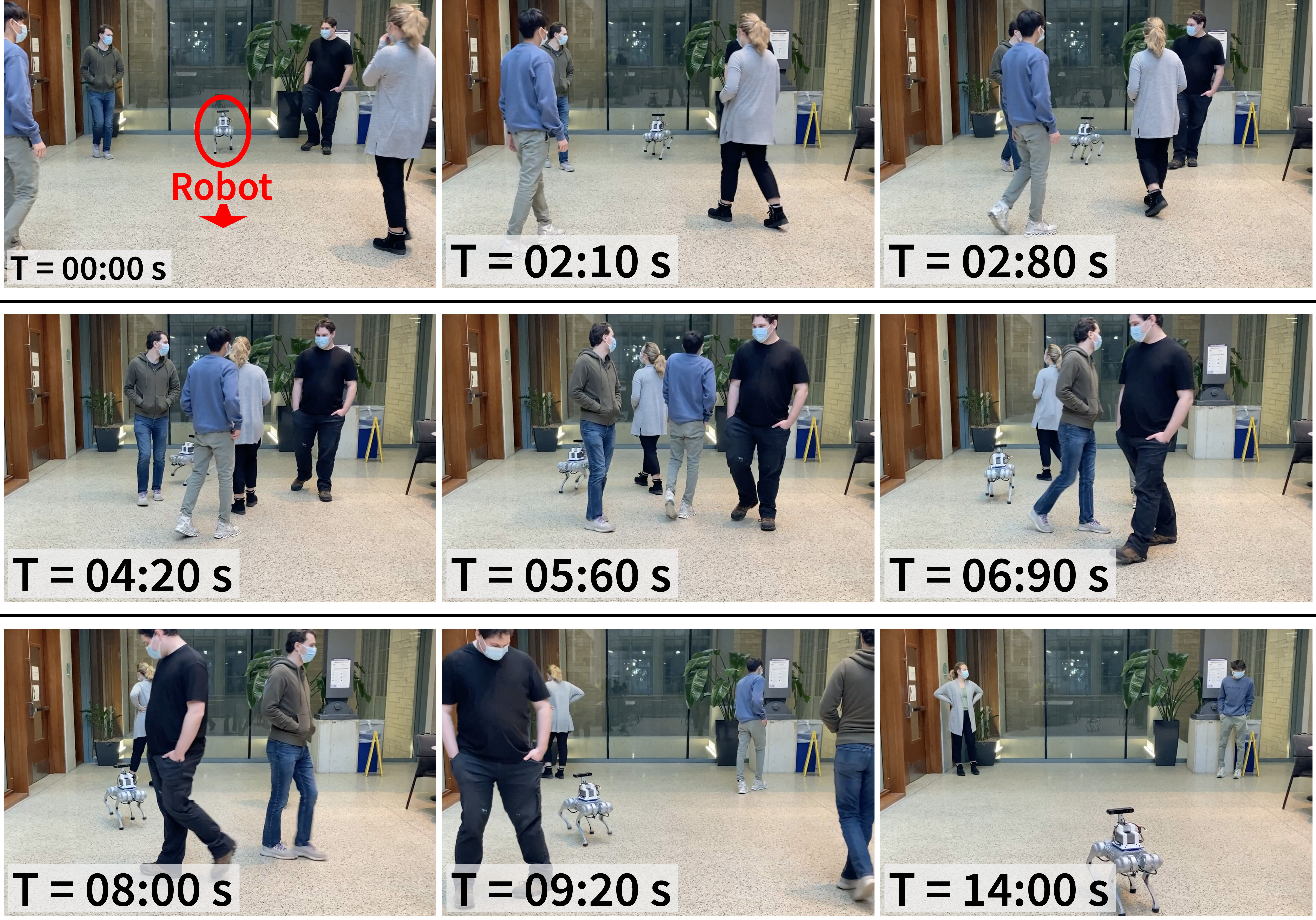}
    \captionsetup{justification=justified}
    \caption{Snapshots from the real-world demonstration with 4 pedestrians. The specification of this demonstration mimics the simulation tests shown in Figure~\ref{fig:simulated_frame_plots}. Videos of the demonstrations are included in the project website: \href{https://sites.google.com/view/brne-crowdnav}{https://sites.google.com/view/brne-crowdnav}.}
    \label{fig:real_world_demonstration_1}
    \vspace{-1em}
\end{figure*}

\begin{figure*}[htbp]
    \centering
    \includegraphics[width=\textwidth]{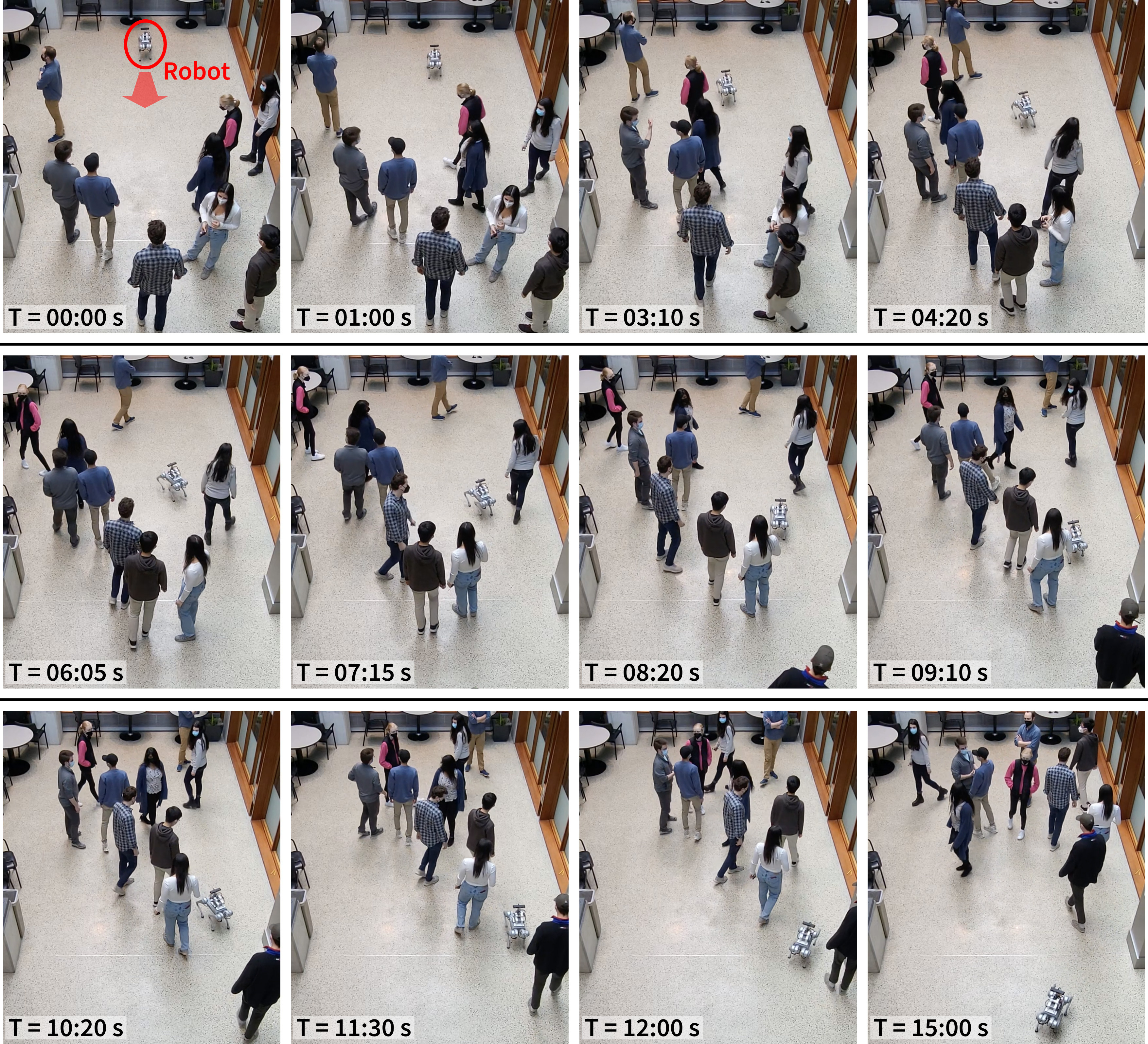}
    \captionsetup{justification=justified}
    \caption{Snapshots from the real-world demonstration with 10 pedestrians. The specification of this demonstration mimics the benchmark tests shown in Figure~\ref{fig:benchmark_frames}. Videos of the demonstrations are included in the project website: \href{https://sites.google.com/view/brne-crowdnav}{https://sites.google.com/view/brne-crowdnav}.}
    \label{fig:real_world_demonstration_2}
    \vspace{-1em}
\end{figure*}

\subsubsection*{Rationale of metrics} We use the four main metrics: (1) The total number of pedestrian collisions; (2) The total number of freezing behaviors, where the robot fails to reach the goal within 60 seconds or runs into environmental obstacles; (3) Path length; (4) Time to goal. Full results with complete benchmark metrics are provided in the appendix. 

\subsubsection*{Results} Table~\ref{table:benchmark_crowd_navigation} shows results of the human dataset crowd navigation experiments. Figure~\ref{fig:benchmark_frames} shows representative snapshots of the benchmark tests. Both our method (\textsf{BRNE}) and the social force model (\textsf{SF}) significantly outperform the rest of the methods on navigation safety. Our method shares the minimal number of collisions with the social force model while being the only method with zero freezing behavior and exhibiting better path efficiency compared with the social force model. More importantly, compared to our meta-planner baseline \textsf{Meta-planner}(Ours), whose path is used by \textsf{BRNE} to generate the nominal mixed strategy, \textsf{BRNE} significantly improves navigation safety, reducing the number of pedestrian collisions from 37 to 1, while sharing nearly identical navigation efficiency as the baseline. As discussed in the next subsection, this one collision incident of our method is due to the pre-processing of the dataset instead of our method's performance. These experimental results show that, when interacting with real human crowds in unstructured environments, our method can improve an existing navigation method's safety to be near-perfect with minimal impact on navigation efficiency. This indicates that our method reaches human-level crowd navigation performance with a meta-planner on safety and navigation efficiency. The results also show great potential for deploying our method on other navigation pipelines. Videos of the human dataset experiments are included in our project website: \href{https://sites.google.com/view/brne-crowdnav}{https://sites.google.com/view/brne-crowdnav}. 

\subsubsection*{Fail case analysis} There is one trial where the robot collides with a pedestrian, here we provide an analysis of the cause, with the snapshots of the collision shown in Figure~\ref{fig:socnavbench_collision_trial}. As we can see, the sudden appearance of a pedestrian right in front of the robot, which is an artifact of the pre-processing of the dataset, leaves the robot with no space and no time to react. Therefore, this collision incident does not reflect our method's navigation safety performance and can be avoided in practice. The video of the failure trial is included in our project website: \href{https://sites.google.com/view/brne-crowdnav}{https://sites.google.com/view/brne-crowdnav}. 

\subsubsection*{Discussion on sim-to-real transfer} While we will demonstrate our BRNE algorithm on robot hardware with real-world humans in the next section, we leave benchmarking our method in real-world crowds for future work, due to the time and financial cost of conducting such benchmarks. Here, we want to discuss the potential gap of sim-to-real transfer based on the human dataset benchmark results. First, since the robot is observing real-world human behavior in the benchmark, the benchmark results can at least verify \emph{which method will fail in the real world}---a method that is dangerous or inefficient in the benchmark will likely remain dangerous or inefficient in the real world. Second, one main challenge of sim-to-real transfer in crowd navigation is the misalignment of human behavior assumptions. In our case, by inferring mixed strategy Nash equilibrium for navigation, the robot assumes the humans are cooperative, which is not necessarily true. However, the human dataset benchmark can serve as a preliminary verification of the influence of the misaligned human behavior assumption. Recall that, even though human agents are not reactive in the benchmark, the robot still believes it is interacting with real humans and observing real human behavior. The numerical efficiency of our method enables fast replanning during navigation, allowing the robot to constantly adjust its plan based on the latest observation, including observations that humans do not follow the assumed cooperative behavior. Furthermore, the robot makes navigation decisions based on mixed strategies---it does not anticipate exact human behavior but instead a distribution of likely human behaviors, which already considers unexpected human reactions. Thus, the fast replanning capability and mixed strategy-based modeling make the performance of our method in the human dataset benchmark more transferable in the real world. Lastly, our method is essentially built on top of an existing meta-planner provided by the benchmark. The modular design of our algorithm enhances the deployment flexibility, which lowers the gap of sim-to-real transfer. For example, our method can be integrated into existing robot navigation frameworks such as the navigation stack from the Robot Operating System~\citep{macenski_marathon_2020,macenski_robot_2022}. This is particularly important for crowd navigation in unstructured real-world environments, and is not addressed by most existing crowd navigation and game-theoretic planning methods. Lastly, even though out of the scope of this paper, we want to point out one important aspect of the sim-to-transfer that is not addressed in the benchmark: the perception aspect of crowd navigation. The benchmark assumes near-perfect robot perception of pedestrian states, but robust perception for both robot localization and pedestrian tracking remains an open challenge and could influence the real-world deployment of our method. 

\subsection{Real-World Hardware Demonstration}
\label{sec:hardware_demonstration}

We demonstrate the practicality of our algorithm in the real world on an untethered robot quadruped with fully on-board perception and computation. The robot successfully completed 10 real-time crowd navigation tasks without collision or freezing behavior, with the number of real human pedestrians varying from 4 to 10. Figure~\ref{fig:real_world_demonstration_1} and Figure~\ref{fig:real_world_demonstration_2} show snapshots from two representative demonstrations. We include videos of all 10 demonstrations on our project website: \href{https://sites.google.com/view/brne-crowdnav}{https://sites.google.com/view/brne-crowdnav}. 

Note that the demonstrations are only for verifying the algorithm's practicality under the conditions and limitations of the real-world environment and robot hardware. They are not intended to be studies for benchmarking the algorithm's performance in real-world human crowds, as such studies are out of the scope of this paper and would require implementing other algorithms in hardware for comparison.

\subsubsection*{System Specification}

We use a Unitree Go1 EDU quadruped robot as the demonstration platform. Even though the robot has holonomic dynamics, we constrain its dynamics to be a differential-drive vehicle to demonstrate the compatibility of our algorithm with a broader range of mobile robots. We use a ZED 2i camera mounted onboard as the perception module. The perception module localizes the robot using visual inertia odometry at 50 Hz and tracks pedestrian position and velocity at 15 Hz. All the computation, including perception and planning with BRNE, is processed entirely by an Nvidia Jetson AGX Orin embedded computer mounted on the robot.

\subsubsection*{Software Implementation}

We implement our BRNE algorithm within the Isaac Robot Operation System (ROS) framework to communicate with the perception module and the robot's low-level controller. We implement the algorithm in Python and used the PyTorch library to accelerate the weight updating step of the algorithm using the low-power GPU of the onboard computer. To maintain the real-time computation speed, we limited the maximum number of agents processed by the algorithm to 5, which allows the algorithm to control the robot at 10 Hz and replan at 5 Hz, with a planning horizon of 2 seconds and 200 trajectory samples per agent. We will release the hardware design and software implementation alongside the algorithm parameters. The algorithm implementation for the hardware demonstration (including parameters) will be released on \href{https://sites.google.com/view/brne-crowdnav}{https://sites.google.com/view/brne-crowdnav}.

\subsubsection*{Demonstration Design} 

The demonstrations are conducted in a $3m$ wide and $9m$ long space inside an indoor atrium. The demonstrations are designed to mimic the specifications of the crowd navigation tests with ORCA agents (such as the ones shown in Figure~\ref{fig:simulated_frame_plots}) and with pedestrian datasets (such as the ones shown in Figure~\ref{fig:benchmark_frames}). The number of pedestrians varies from $4$ to $10$. Our demonstrations follow common design principles in other crowd navigation works, such as \citet{mavrogiannis_social_2018} and \citet{mavrogiannis_winding_2023}. To mimic the test with ORCA agents, pedestrians were instructed to go to the designated goals in a circle. To mimic the pedestrian dataset benchmark, pedestrians were instructed to move freely in the space. Across all demonstrations, the pedestrians are instructed to ``move with normal walking speed and treat the robot as a walking person or dog''. The navigation task for the robot is to move safely and efficiently from one side of the room to the other, with a traveling distance varying between $6m$ to $8m$. The robot moves at a nominal speed of $0.5 m/s$.

\subsubsection*{Discussion on Hardware Demonstration}
The first challenge encountered during the hardware demonstration is specifying the risk function (Definition~\ref{def:risk_function}). We model the risk function as a logistic function based on the distance between two agents. The parameters of this function, which dictate how an agent's willingness to be near others decreases with distance, are set manually through trial and error. Future work could involve learning the risk function from pedestrian datasets and dynamically adjusting its parameters based on real-time observations. A second challenge arises from inconsistent pedestrian observations due to occlusion. These inconsistencies cause agents to appear and disappear unexpectedly, leading to inaccuracies in pedestrian velocity estimation. To address this, we apply an iterative-closest-point method to associate pedestrians across frames, which reduces this issue but may still result in oscillations in velocity estimates, affecting the fidelity of nominal pedestrian strategies. Future work could explore data-driven filters to further stabilize pedestrian tracking.

\section{Conclusion} \label{sec:conclusion_discussion}

In this work, we propose a computation-efficient mixed strategy Nash equilibrium model for crowd navigation. Mixed strategy Nash equilibrium provides a high-level cooperation model for the robot to plan actions that leverage human cooperation for collision avoidance while maintaining the uncertainty in human behavior. Despite the general hardness of computing mixed strategy Nash equilibrium, we achieve real-time inference speed on a laptop CPU and a low-power embedded computer by establishing the formal connection between mixed strategy Nash equilibrium and a simple iterative Bayesian update scheme. We name the proposed model Bayesian Recursive Nash Equilibrium (BRNE). We further develop a model predictive crowd navigation framework using Gaussian processes to bootstrap agents' nominal mixed strategies. The proposed model can be incorporated into existing navigation frameworks to navigate alongside humans in unstructured environments with static obstacles. Our experiment results show that our BRNE model significantly improves the safety of a human-unaware planner without compromising navigation efficiency in the human dataset benchmark, reaching human-level performance. Compared to other crowd navigation algorithms, our model consistently outperforms them in both safety and navigation efficiency. Lastly, we demonstrate the practicality of our model on an untethered robot quadruped for real-time crowd navigation with fully onboard perception and computation.

Beyond the technical contributions, our work also provides valuable insights into how game theory models can be applied to robotics applications. Most, if not all, existing game-theoretic planning algorithms are designed with a top-down approach: the planner aims to solve a generalized game, with the player objectives being specified later for different applications. While this approach can be applied to arbitrary games, the limitation of the top-down approach is the hardness of solving a generalized game, which is often intractable or too expensive to compute. In this work, we take a bottom-up approach to apply game theory models. We use mixed strategy Nash equilibrium as a high-level principle to design a specific behavioral model for real-time crowd navigation with limited computation resources. This bottom-up approach allows us to utilize the analytical power of mixed strategy Nash equilibrium while maintaining a sufficient computation load for real-time robot navigation. We consider our bottom-up approach to make an orthogonal contribution, with respect to existing top-down approaches, to the principles of designing game-theoretic planners. 

There are several limitations of the proposed crowd navigation framework. First, our framework does not support extra constraints on decision-making, such as guarantees on static obstacle avoidance. Even though we have zero static obstacle collision in the human dataset benchmark, this is achieved using a collision-free meta-planner as the nominal strategy. Our method itself does not guarantee the robot from a collision with environmental obstacles. Second, the sampling scheme limits the practical performance of our crowd navigation framework. We proposed two sampling strategies for approximating the mixed strategy Nash equilibrium. The rejection sampling-based strategy generates a more accurate posterior but is too slow for real-time decision-making. The importance-sampling-based strategy supports real-time posterior estimation, but its performance is limited by the support of initial samples and suffers from common issues faced by particle filters, such as sample degeneracy. Third, throughout the experiments, we assume the localization for the robot and perception of pedestrians are given, while perception in real human crowds is still an open challenge. Our ongoing work involves benchmarking the algorithm on the hardware in the real world, and investigating how sensor noise would affect the algorithm's performance. 

Mixed strategy Nash equilibrium paves the way for future works to achieve truly adaptive crowd navigation in varying environments, as it provides extra information enhancing the robot’s adaptivity during the run-time. We want to point out an exciting future direction stemming from our model, which is adaptive or contingent control synthesis based on mixed strategy Nash equilibrium. In this work, we simply synthesize robot control signals to follow the mean of the converged mixed strategy of the robot. This approach does not fully utilize the information encoded in mixed strategy Nash equilibrium. For example, mixed strategy predictions can be a prior distribution when existing information is insufficient for precise behavior prediction. Based on this prior, the robot could quickly refine its prediction or perform contingency planning with new online measurements.

\section*{\normalsize Acknowledgements}
The authors would like to thank Matthew Elwin and Maia Traub for their assistance with the hardware demonstration.

\section*{\normalsize Funding} 

This material is supported by the Honda Research Institute Grant HRI-001479. Any opinions, findings, conclusions, or recommendations expressed in this material are those of the authors and do not necessarily reflect the views of the aforementioned institutions.

\section*{\normalsize Declaration of conflicting interests}

The author(s) declared no potential conflicts of interest with respect to the research, authorship, and/or publication of this article.

\bibliographystyle{SageH}
\bibliography{references}


\section*{Appendix A: Proofs}

Theorem~\ref{theorem:twoagent_posterior_optimality} and Theorem~\ref{theorem:multiagent_posterior_optimality} are to prove the convergence of Algorithm~\ref{algo:two_agent_update} and Algorithm~\ref{algo:multi_agent_update}. Again, we start with the two-agent scenario and extend the result to the multi-agent scenario.

\begin{theorem} \label{theorem:two_agent_convergence}
    The sequence of $\{(p_{i}^{[k]}, p_{j}^{[k]})\}_k$ from the iterations of Algorithm \ref{algo:two_agent_update} is a convergent sequence.
\end{theorem}
\begin{proof}
    We prove the sequence's convergence through the monotone convergence theorem, by proving that the sequence monotonically decreases on a function with a finite lower bound. We start with the monotone decrease result. 

    Based on the global optimality result in Theorem~\ref{theorem:twoagent_posterior_optimality}, for each iteration in Algorithm~\ref{algo:two_agent_update}, after computing agent $i$'s posterior belief $p_i^{[k+1]}$, we have the following inequality:
    \begin{align}
        & \mathbb{E}_{p_{i}^{[k+1]}, p_{j}^{[k]}}[r] {+} D(p_{i}^{[k+1]}\Vert p^\prime_{i}) 
    \leq \mathbb{E}_{p_{i}^{[k]}, p_{j}^{[k]}}[r] {+} D(p_{i}^{[k]}\Vert p^\prime_{i}).
    \end{align} Adding the term $D(p_{j}^{[k]}\Vert p^\prime_{j})$ to both sides gives us:
    \begin{align}
        & \mathbb{E}_{p_{i}^{[k+1]}, p_{j}^{[k]}}[r] + D(p_{i}^{[k+1]}\Vert p^\prime_{i}) + D(p_{j}^{[k]}\Vert p^\prime_{j}) \nonumber \\
        \leq \quad & \mathbb{E}_{p_{i}^{[k]}, p_{j}^{[k]}}[r] + D(p_{i}^{[k]}\Vert p^\prime_{i}) + D(p_{j}^{[k]}\Vert p^\prime_{j}). \label{eq:temp_eq1}
    \end{align} Then, given $p_i^{[k+1]}$, updating $p_j^{[k]}$ gives us the following inequality:
    \begin{align}
        & \mathbb{E}_{p_{i}^{[k+1]}, p_{j}^{[k+1]}}[r] {+} D(p_{j}^{[k+1]}\Vert p^\prime_{j}) \leq \mathbb{E}_{p_{i}^{[k+1]}, p_{j}^{[k]}}[r] {+} D(p_{j}^{[k]}\Vert p^\prime_{j}).
    \end{align} Adding $D(p_{i}^{[k+1]}\Vert p^\prime_{i})$ to both sides gives us:
    \begin{align}
        & \mathbb{E}_{p_{i}^{[k+1]}, p_{j}^{[k+1]}}[r] + D(p_{i}^{[k+1]}\Vert p^\prime_{i}) + D(p_{j}^{[k+1]}\Vert p^\prime_{j}) \nonumber \\
        \leq \quad & \mathbb{E}_{p_{i}^{[k+1]}, p_{j}^{[k]}}[r] + D(p_{i}^{[k+1]}\Vert p^\prime_{i}) + D(p_{j}^{[k]}\Vert p^\prime_{j}). \label{eq:temp_eq2}
    \end{align} By applying the chain rule to inequality (\ref{eq:temp_eq1}) and (\ref{eq:temp_eq2}), we have:
    \begin{align}
        & \mathbb{E}_{p_{i}^{[k+1]}, p_{j}^{[k+1]}}[r] + D(p_{i}^{[k+1]}\Vert p^\prime_{i}) + D(p_{j}^{[k+1]}\Vert p^\prime_{j}) \nonumber \\
        \leq \quad & \mathbb{E}_{p_{i}^{[k]}, p_{j}^{[k]}}[r] + D(p_{i}^{[k]}\Vert p^\prime_{i}) + D(p_{j}^{[k]}\Vert p^\prime_{j}). 
    \end{align} The above inequality means the sequence $\{(p_{i}^{[k]}, p_{j}^{[k]})\}_k$ from Algorithm \ref{algo:two_agent_update} monotonically decreases the function:
    \begin{align}
        F(p_i,p_j) = \mathbb{E}_{p_{i}, p_{j}}[r] + D(p_{i}\Vert p^\prime_i) + D(p_{j}\Vert p^\prime_j). \label{eq:temp_obj_func}
    \end{align} Based on the non-negativity of KL-divergence, this function (\ref{eq:temp_obj_func}) has a finite lower-bound. Thus, based on the monotone convergence theorem, the sequence $\{(p_{i}^{[k]}, p_{j}^{[k]})\}_k$ is convergent under (\ref{eq:convergence_measure}). \qed
\end{proof}

Theorem~\ref{theorem:two_agent_convergence} can be extended to the multi-agent scenario in Algorithm~\ref{algo:multi_agent_update}. 
\begin{theorem} \label{theorem:multi_agent_convergence}
    The sequence of $\{(p_{1}^{[k]}, \dots, p_{N}^{[k]})\}_k$ from the iterations of Algorithm \ref{algo:multi_agent_update} is a convergent sequence.
\end{theorem}
\begin{proof}
    At iteration $k$, for an arbitrary agent with index $a\in\mathcal{I}$, based on the global optimality result in Theorem~\ref{theorem:multiagent_posterior_optimality}, we have the following inequality:
    \begin{align}
        & {\sum_{i=1}^{a-1}} \mathbb{E}_{p_{a}^{[k+1]}, p_{i}^{[k+1]}}[r] {+} {\sum_{i=a+1}^{N}} \mathbb{E}_{p_{a}^{[k+1]}, p_{i}^{[k]}}[r] {+} D(p_{a}^{[k+1]}\Vert p^\prime_{a}) \nonumber \\
        \leq \quad & {\sum_{i=1}^{a-1}} \mathbb{E}_{p_{a}^{[k]}, p_{i}^{[k+1]}}[r] {+} {\sum_{i=a+1}^{N}} \mathbb{E}_{p_{a}^{[k]}, p_{i}^{[k]}}[r] {+} D(p_{a}^{[k]}\Vert p^\prime_{a}).
    \end{align} By adding the following nested summation to both sides of the inequality:
    \begin{align}
        & \sum_{i=1}^{a-1}\sum_{j=i+1}^{a-1} \mathbb{E}_{p_i^{[k+1]},p_j^{[k+1]}}[r] + \sum_{i=a+1}^{N}\sum_{j=a+1}^{N} \mathbb{E}_{p_i^{[k]},p_j^{[k]}}[r] \nonumber \\
        & + \sum_{i=1}^{a-1} D(p_{i}^{[k+1]}\Vert p^\prime_{i}) + \sum_{i=a+1}^{N} D(p_{i}^{[k]}\Vert p^\prime_{i}) ,
    \end{align} we have the left-hand side as:
    \begin{align}
        & {\sum_{i=1}^{a-1}} \mathbb{E}_{p_{a}^{[k+1]}, p_{i}^{[k+1]}}[r] {+} {\sum_{i=a+1}^{N}} \mathbb{E}_{p_{a}^{[k+1]}, p_{i}^{[k]}}[r] {+} D(p_{a}^{[k+1]}\Vert p^\prime_{a}) \nonumber \\
        & + {\sum_{i=1}^{a-1}\sum_{j=i+1}^{a-1}} \mathbb{E}_{p_i^{[k+1]},p_j^{[k+1]}}[r] {+} {\sum_{i=a+1}^{N}\sum_{j=a+1}^{N}} \mathbb{E}_{p_i^{[k]},p_j^{[k]}}[r] \nonumber \\
        & + \sum_{i=1}^{a-1} D(p_{i}^{[k+1]}\Vert p^\prime_{i}) + \sum_{i=a+1}^{N} D(p_{i}^{[k]}\Vert p^\prime_{i})  \\
        = \quad & \sum_{i=1}^{a-1}\sum_{j=i+1}^{a} \mathbb{E}_{p_i^{[k+1]},p_j^{[k+1]}}[r] + {\sum_{i=a+1}^{N}} \mathbb{E}_{p_{a}^{[k+1]}, p_{i}^{[k]}}[r] \nonumber \\
        & + \sum_{i=a+1}^{N}\sum_{j=a+1}^{N} \mathbb{E}_{p_i^{[k]},p_j^{[k]}}[r] \nonumber \\
        & + \sum_{i=1}^{a} D(p_{i}^{[k+1]}\Vert p^\prime_{i}) + \sum_{i=a+1}^{N} D(p_{i}^{[k]}\Vert p^\prime_{i}), 
    \end{align} and the right hand side as:
    \begin{align}
        & {\sum_{i=1}^{a-1}} \mathbb{E}_{p_a^{[k]},p_i^{[k+1]}}[r] {+} {\sum_{i=a+1}^{N}} \mathbb{E}_{p_a^{[k]},p_i^{[k]}}[r] {+} D(p_{a}^{[k]}\Vert p^\prime_{a}) \nonumber \\
        & + {\sum_{i=1}^{a-1}\sum_{j=i+1}^{a-1}} \mathbb{E}_{p_i^{[k+1]},p_j^{[k+1]}}[r] {+} {\sum_{i=a+1}^{N}\sum_{j=a+1}^{N}} \mathbb{E}_{p_i^{[k]},p_j^{[k]}}[r] \nonumber \\
        & + \sum_{i=1}^{a-1} D(p_{i}^{[k+1]}\Vert p^\prime_{i}) + \sum_{i=a+1}^{N} D(p_{i}^{[k]}\Vert p^\prime_{i}) \\
        = \quad & {\sum_{i=1}^{a-1}\sum_{j=i+1}^{a-1}} \mathbb{E}_{p_i^{[k+1]},p_j^{[k+1]}}[r] {+} {\sum_{i=1}^{a-1}} \mathbb{E}_{p_a^{[k]},p_i^{[k+1]}}[r] \nonumber \\
        & + {\sum_{i=a+1}^{N}\sum_{j=a}^{N-1}} \mathbb{E}_{p_i^{[k]},p_j^{[k]}}[r] \nonumber \\
        & + \sum_{i=1}^{a-1} D(p_{i}^{[k+1]}\Vert p^\prime_{i}) + \sum_{j=i+1}^{N} D(p_{i}^{[k]}\Vert p^\prime_{i}). 
    \end{align} Combining both sides gives us the following inequality:
    \begin{align}
        & \sum_{i=1}^{a-1}\sum_{j=i+1}^{a} \mathbb{E}_{p_i^{[k+1]},p_j^{[k+1]}}[r] + {\sum_{i=a+1}^{N}} \mathbb{E}_{p_{a}^{[k+1]}, p_{i}^{[k]}}[r] \nonumber \\
        & + \sum_{i=a+1}^{N}\sum_{j=a+1}^{N} \mathbb{E}_{p_i^{[k]},p_j^{[k]}}[r] \nonumber \\
        & + \sum_{i=1}^{a} D(p_{i}^{[k+1]}\Vert p^\prime_{i}) + \sum_{i=a+1}^{N} D(p_{i}^{[k]}\Vert p^\prime_{i}) \nonumber \\
        \leq \quad & {\sum_{i=1}^{a-1}\sum_{j=i+1}^{a-1}} \mathbb{E}_{p_i^{[k+1]},p_j^{[k+1]}}[r] {+} {\sum_{i=1}^{a-1}} \mathbb{E}_{p_a^{[k]},p_i^{[k+1]}}[r] \nonumber \\
        & + {\sum_{i=a+1}^{N}\sum_{j=a}^{N-1}} \mathbb{E}_{p_i^{[k]},p_j^{[k]}}[r] \nonumber \\
        & + \sum_{i=1}^{a-1} D(p_{i}^{[k+1]}\Vert p^\prime_{i}) + \sum_{j=i+1}^{N} D(p_{i}^{[k]}\Vert p^\prime_{i}).
    \end{align} By iterating the agent index $a$ from $1$ to $N$, based on the chain rule of inequality, we have the following chain of inequalities:
    \begin{align}
        & \sum_{i=1}^{N}\sum_{j=i+1}^{N} \mathbb{E}_{p_i^{[k+1]},p_j^{[k+1]}}[r] + \sum_{i=1}^{N} D(p_{i}^{[k+1]}\Vert p^\prime_{i}) \nonumber \\
        \leq \quad & \cdots \nonumber \\
        \leq \quad & {\sum_{i=1}^{a-1}\sum_{j=i+1}^{a-1}} \mathbb{E}_{p_i^{[k+1]},p_j^{[k+1]}}[r] {+} {\sum_{i=1}^{a-1}} \mathbb{E}_{p_a^{[k]},p_i^{[k+1]}}[r] \nonumber \\
        & + {\sum_{i=a+1}^{N}\sum_{j=a}^{N-1}} \mathbb{E}_{p_i^{[k]},p_j^{[k]}}[r] + \sum_{i=1}^{N} \nonumber \\
        & + \sum_{i=1}^{a-1} D(p_{i}^{[k+1]}\Vert p^\prime_{i}) + \sum_{j=i+1}^{N} D(p_{i}^{[k]}\Vert p^\prime_{i}) \\
        \leq \quad & \cdots \nonumber \\
        \leq \quad & \sum_{i=1}^{N}\sum_{j=i+1}^{N} \mathbb{E}_{p_i^{[k]},p_j^{[k]}}[r] + \sum_{i=1}^{N} D(p_{i}^{[k]}\Vert p^\prime_{i}),
    \end{align} which means the sequence $\{(p_{0}^{[k]}, \dots, p_{N-1}^{[k]})\}_k$ monotonically decreases the function:
    \begin{align}
        F(p_1,\dots,p_N) = \sum_{i=1}^{N}\sum_{j=i+1}^{N} \mathbb{E}_{p_i,p_j}[r] + \sum_{i=1}^{N} D(p_{i}\Vert p^\prime_{i}). \label{eq:temp_obj_func2}
    \end{align} Based on the non-negativity of KL-divergence, this function (\ref{eq:temp_obj_func2}) is lower-bounded. Thus, based on the monotone convergence theorem, the sequence $\{(p_{1}^{[k]}, \dots, p_{N}^{[k]})\}_k$ is convergent under (\ref{eq:convergence_measure}). \qed
\end{proof}

\subsection{Proof for Theorem~\ref{theorem:nash_equilibrium}}
\begin{proof}
    Theorem~\ref{theorem:multi_agent_convergence} has proved the convergence of the mixed strategy sequence under (\ref{eq:convergence_measure}). By contradiction, we can prove the limit point is a global Nash equilibrium of the mixed strategy game (\ref{eq:general_sum_player_definition}). Denote the sequence converges to the limit point $(p_1^*, \dots, p_N^*)$, assume the limit point is not a global Nash equilibrium of (\ref{eq:general_sum_player_definition}), then
    \begin{align}
        & \exists i\in\mathcal{I}, \exists p_i(s)\in\mathcal{P}, \text{s.t. } \nonumber \\
        & \mathbb{E}_{p_i,p_{/i}^*}[r] + D(p_i\Vert p^\prime_i) < \mathbb{E}_{p_i^*,p_{/i}^*}[r] + D(p_i^*\Vert p^\prime_i) \label{eq:contradict_ineq}.
    \end{align} Since the mixed strategy sequence from Algorithm~\ref{algo:multi_agent_update} monotically decreases the lower-bounded function (\ref{eq:temp_obj_func2}) and the right hand side of (\ref{eq:contradict_ineq}) is part of the summation in (\ref{eq:temp_obj_func2}), the inequality (\ref{eq:contradict_ineq}) indicates $(p_1^*, \dots, p_N^*)$ is not a limit point under (\ref{eq:convergence_measure}), which contradicts the assumption, thus completes the proof. \qed
\end{proof}

\subsection{Proof for Theorem~\ref{theorem:risk_reduction}}
\begin{proof}
    Recall in the proof for Theorem~\ref{theorem:multi_agent_convergence}, the sequence of $\{(p_{1}^{[k]}, \dots, p_{N}^{[k]})\}$ from the iterations of Algorithm \ref{algo:multi_agent_update} monotonically decreases the function:
    \begin{align}
        \sum_{i=1}^{N}\sum_{j=i+1}^{N} \mathbb{E}_{p_i^{[k]},p_j^{[k]}}[r] + \sum_{i=1}^{N} D(p_{i}^{[k]}\Vert p^\prime_{i}).
    \end{align} Since $p_i^{[0]}=p^\prime_i$, since KL-divergence is zero between two identical distributions, we have the following inequality:
    \begin{align}
        & \sum_{i=1}^{N}\sum_{j=i+1}^{N} \mathbb{E}_{p_i^*,p_j^*}[r] + \sum_{i=1}^{N} D(p_{i}^*\Vert p^\prime_{i}) \nonumber \\
        \leq \quad & \sum_{i=1}^{N}\sum_{j=i+1}^{N} \mathbb{E}_{p^\prime_i,p^\prime_j}[r] + \sum_{i=1}^{N} D(p^\prime_{i}\Vert p^\prime_{i}) \\
        = & \sum_{i=1}^{N}\sum_{j=i+1}^{N} \mathbb{E}_{p^\prime_i,p^\prime_j}[r] + 0,
    \end{align} which completes the proof. \qed
\end{proof}

\newpage

\section*{Appendix B: Pseudocode}
\begin{algorithm} \label{algo:brne_reject}
    \caption{BRNE navigation (\emph{rejection} sampling)}
    \begin{algorithmic}[1]
        \Procedure{BRNE\_Nav}{$\mathbf{p}^\prime_{1}, \dots, \mathbf{p}^\prime_{N}, \gamma$} \Comment{$\gamma>1$.}
        \State $k \gets 0$ \Comment{$k$ is the negotiation step.}
        \For{$i\in[1,N]$} \Comment{Number of agents}
            \For{$j\in[1,M]$} \Comment{Number of samples$/$agent}
                \State $\mathbf{s}_{i,j}^{[k]} \gets \mathbf{s}^\prime_{i,j}$ \Comment{Sample from nominal strategy}
            \EndFor
        \EndFor
        \While{convergence criterion not met}
            \For{$i\in[1,N]$}
                \State $\mathbf{p}^{[k]}_{/i} \gets \left({\bigcup_{a=1}^{i-1}} \mathbf{p}^{[k]}_a\right) {\cup} \left(\bigcup_{a=i+1}^{N} \mathbf{p}^{[k+1]}_a\right)$
                \For{$j\in[1,M]$} \Comment{Rejection sampling}
                    \State $\mathbf{s}_{new} \gets$ Draw new sample from $p^\prime_i$
                    \State $\omega \gets z(\mathbf{s}_{new}\vert \mathbf{p}_{/i}^{[k]})$
                    \State $u \gets Uniform(0, 1)$
                    \While{$\gamma\cdot u \geq \omega$}
                        \State $\mathbf{s}_{new} \gets$ Draw new sample from $p^\prime_i$
                        \State $\omega \gets z(\mathbf{s}_{new}\vert \mathbf{p}_{/i}^{[k]})$
                        \State $u \gets Uniform(0, 1)$
                    \EndWhile
                    \State $\mathbf{s}_{i,j}^{[k]} \gets \mathbf{s}_{new}$
                \EndFor
            \EndFor
            \State $k \gets k+1$
        \EndWhile
        \State \textbf{return} $\mathbf{p}_{1}^{[k]}, \dots, \mathbf{p}_{N}^{[k]}$
        \EndProcedure
    \end{algorithmic}
\end{algorithm}

\begin{algorithm} \label{algo:brne_importance}
    \caption{BRNE navigation (\emph{importance} sampling)}
    \begin{algorithmic}[1]
        \Procedure{BRNE\_Nav}{$\mathbf{p}^\prime_{1}, \dots, \mathbf{p}^\prime_{N}$}
        \State $k \gets 0$ \Comment{$k$ is the negotiation step.}
        \For{$i\in[1,N]$} \Comment{Number of agents}
            \For{$j\in[1,M]$} \Comment{Number of samples$/$agent}
                \State $\mathbf{w}_{i,j}^{[k]} \gets 1$ \Comment{Initialize sample weights}
            \EndFor
        \EndFor
        \While{convergence criterion not met}
            \For{$i\in[1,N]$}
                \For{$j\in[1,M]$} \Comment{Importance sampling}
                    \State $\displaystyle \mathbf{w} {\gets} {\frac{1}{M}} {\sum_{a=1}^{i-1}} {\sum_{b=1}^{M}} \mathbf{w}_{a,b}^{[k]} r(\mathbf{s}_{i,j}^{\prime}, \mathbf{s}_{a,b}^{\prime})$ 
                    \State $\displaystyle \mathbf{w} {\gets} \mathbf{w} {+} {\frac{1}{M}}{\sum_{a=i+1}^{N}}{\sum_{b=1}^{M}} {\mathbf{w}_{a,b}^{[k+1]}}{r(\mathbf{s}_{i,j}^{\prime}, \mathbf{s}_{a,b}^{\prime})}$ 
                    \State $\mathbf{w}_{i,j}^{[k+1]} \gets \mathbf{w}$
                \EndFor
                \State{Normalize $\mathbf{w}_{i}^{[k+1]} = [ \mathbf{w}_{i,1}^{[k+1]}, \dots, \mathbf{w}_{i,M}^{[k+1]} ]$}
            \EndFor
            \State $k \gets k+1$
        \EndWhile
        \State \textbf{return} $\mathbf{w}_{1}^{[k]}, \dots, \mathbf{w}_{N}^{[k]}$
        \EndProcedure
    \end{algorithmic}
\end{algorithm}

\begin{table*}[ht!]
    \centering
    \captionsetup{justification=centering}
    \caption{Meta statistics of human dataset crowd navigation experiments.}
    \setlength{\tabcolsep}{18.0pt}
    \begin{tabular}{|c|cccc|}
        \toprule
        Algorithm & \begin{tabular}[c]{@{}c@{}}Overall\\Success rate\end{tabular} & \begin{tabular}[c]{@{}c@{}}Failure cases\\(T/PC/EC)\end{tabular} & \begin{tabular}[c]{@{}c@{}}Total pedestiran\\collisions\end{tabular} & \begin{tabular}[c]{@{}c@{}}Planning wall time\\per episode (s)\end{tabular}\\
        \midrule
        \textsf{BRNE}(Ours) & $32/33$ & $(0/1/0)$ & $1$ & $20.39\pm 6.93$ \\
        \textsf{Baseline}(Ours) & $10/33$ & $(0/23/0)$ & $37$ & $9.07\pm 2.12$ \\
        \midrule
        \textsf{SF} & $32/33$ & $(1/0/0)$ & $1$ & $18.23\pm 7.41$ \\
        \textsf{ORCA} & $24/33$ & $(1/8/0)$ & $15$ & $48.84\pm 23.06$ \\
        \textsf{CADRL} & $18/33$ & $(0/14/1)$ & $40$ & $46.78\pm 21.38$ \\
        \textsf{Baseline} & $9/33$ & $(1/23/0)$ & $64$ & $51.12\pm16.21$ \\
        \bottomrule 
    \end{tabular}
    \label{table:appendex_table_1}
\end{table*}

\begin{table*}[ht!]
    \centering
    \captionsetup{justification=centering}
    \caption{Path quality results of human dataset crowd navigation experiments.}
    \setlength{\tabcolsep}{10.0pt}
    \begin{tabular}{|c|ccccc|}
        \toprule
        Algorithm & \begin{tabular}[c]{@{}c@{}}Path length (m)\end{tabular} & \begin{tabular}[c]{@{}c@{}}Path length ratio\end{tabular} & \begin{tabular}[c]{@{}c@{}}Goal traversal ratio\end{tabular} & \begin{tabular}[c]{@{}c@{}}Path irregularity\\(radians)\end{tabular} & \begin{tabular}[c]{@{}c@{}}Path traversal\\time (s)\end{tabular} \\
        \midrule
        \textsf{BRNE}(Ours) & $16.56\pm 3.85$ & $1.06\pm 0.05$ & $0.04$ & $1.58\pm 1.04$ & $18.95\pm 4.81$ \\
        \textsf{Baseline}(Ours) & $15.42\pm 3.71$ & $0.99\pm 0.03$ & $0.02$ & $1.58\pm 1.05$ & $17.66\pm 4.28$ \\
        \midrule
        \textsf{SF} & $17.25\pm 4.05$ & $1.15\pm 0.23$ & $0.52$ & $1.66\pm 0.95$ & $15.93\pm 4.17$ \\
        \textsf{ORCA} & $17.66\pm 5.22$ & $1.17\pm 0.26$ & $0.21$ & $1.56\pm 1.01$ & $22.06\pm 7.78$ \\
        \textsf{CADRL} & $15.70\pm 3.72$ & $1.04\pm 0.05$ & $0.51$ & $1.68\pm 1.04$ & $15.14\pm 4.21$ \\
        \textsf{Baseline} & $15.88\pm 3.57$ & $1.05\pm 0.11$ & $0.09$ & $1.65\pm 1.02$ & $16.08\pm 3.73$ \\
        \bottomrule                                                  
    \end{tabular}
    \label{table:appendex_table_2}
\end{table*}

\begin{table*}[ht!]
    \centering
    \captionsetup{justification=centering}
    \caption{Motion quality results of human dataset crowd navigation experiments.}
    \setlength{\tabcolsep}{13.0pt}
    \begin{tabular}{|c|cccc|}
        \toprule
        Algorithm & \begin{tabular}[c]{@{}c@{}}Average speed ($m/s$)\\(max=$1.2m/s$)\end{tabular} & \begin{tabular}[c]{@{}c@{}}Average energy\\expenditure (J)\end{tabular} & \begin{tabular}[c]{@{}c@{}}Average\\acceleration ($m/s^2$)\end{tabular} & \begin{tabular}[c]{@{}c@{}}Average jerk ($m/s^3$)\end{tabular}\\
        \midrule
        \textsf{BRNE}(Ours) & $0.88\pm 0.31$ & $327.97\pm 79.02$ & $3.90\pm 7.76$ & $151.44\pm 257.73$ \\
        \textsf{Baseline}(Ours) & $0.88\pm 0.31$ & $305.05\pm 74.25$ & $3.91\pm 7.84$ & $155.04\pm 260.59$ \\
        \midrule
        \textsf{SF} & $1.09\pm 0.27$ & $398.33\pm 95.45$ & $0.39\pm 1.43$ & $2.70\pm 28.23$ \\
        \textsf{ORCA} & $0.80\pm 0.27$ & $315.03\pm 89.89$ & $0.31\pm 1.26$ & $6.28\pm 28.85$ \\
        \textsf{CADRL} & $1.04\pm 0.36$ & $367.58\pm 87.24$ & $0.93\pm 2.91$ & $31.68\pm 87.71$ \\
        \textsf{Baseline} & $0.99\pm 0.42$ & $370.88\pm 84.55$ & $4.81\pm 9.07$ & $180.86\pm 303.88$ \\
        \bottomrule                                                  
    \end{tabular}
    \label{table:appendex_table_3}
\end{table*}

\section*{Appendix C: Complete human dataset benchmark results}
Complete results from the human dataset benchmark are shown in Table~\ref{table:appendex_table_1}, Table~\ref{table:appendex_table_2}, and Table~\ref{table:appendex_table_3}. Definitions of the metrics can be found in~\citep{biswas_socnavbench_2022}.

\end{document}